\documentclass[a4paper,10pt,oneside]{amsart}

\usepackage[english]{babel}
\usepackage[T1]{fontenc}
\usepackage{amsmath,amssymb,amsthm,amsfonts}
\usepackage[backgroundcolor=red!97!green!97!blue!10!]{todonotes}
\usepackage{mathrsfs}
\usepackage[all]{xy}
\usepackage{overpic}
\usepackage{graphicx}
\usepackage{subfigure}
\usepackage{booktabs}
\usepackage{hyperref}
\usepackage[shortlabels]{enumitem}

\theoremstyle{plain}
\newtheorem{lemma}{Lemma}[section]
\newtheorem{conjlemma}{\textbf{Conjectured Lemma}}
\newtheorem{proposition}[lemma]{\textbf{Proposition}}
\newtheorem{theorem}[lemma]{\textbf{Theorem}}

\theoremstyle{definition}
\newtheorem{definition}[lemma]{\textbf{Definition}}
\newtheorem{example}[lemma]{\textbf{Example}}

\newtheorem*{fact}{\textbf{Fact}}
\theoremstyle{remark}
\newtheoremstyle{indenteddefinition}{8pt}{8pt}{\small 
\addtolength{\leftskip}{2.5em}}{-2.5em}{\itshape}{.}{ }{}
\theoremstyle{indenteddefinition}
\newtheorem{remark}[lemma]{Remark}

\newcommand{\R}{\mathbb{R}}
\newcommand{\C}{\mathbb{C}}

\newcommand{\p}{\mathbb{P}}
\newcommand{\SE}{\mathrm{SE}_3}
\newcommand{\SO}{\mathrm{SO}_3}
\renewcommand\epsilon{\varepsilon}
\renewcommand\tilde{\widetilde}

\newcommand{\tth}{\thinspace}
\newcommand{\mcal}{\mathcal}
\newcommand{\mscr}{\mathscr}
\newcommand{\go}[1]{{\sf #1}}

\def\phm{{\hphantom{-}}}

\title{Liaison Linkages}

\author{Matteo Gallet}

\address{
Matteo Gallet \\
Johann Radon Institute for Computational and Applied Mathematics (RICAM) \\
Austrian Academy of Sciences \\
Altenberger Stra\ss e 69 \\
4040 Linz, Austria.}

\email{matteo.gallet@ricam.oeaw.ac.at}

\author{Georg Nawratil}

\address{
Georg Nawratil \\
Institute of Discrete Mathematics and Geometry\\
Vienna University of Technology \\
Wiedner Hauptstrasse 8-10/104 \\
1040 Vienna, Austria.}

\email{nawratil@geometrie.tuwien.ac.at}

\author{Josef Schicho}

\address{
Josef Schicho \\
Research Institute for Symbolic Computation \\
Johannes Kepler University \\
Altenberger Stra\ss e 69 \\
4040 Linz, Austria.}

\email{josef.schicho@risc.jku.at}

\begin{document}

\begin{abstract}
The complete classification of hexapods --- also known as Stewart Gough 
platforms --- of mobility one is still open. To tackle this problem, we can 
associate to each hexapod of mobility one an algebraic curve, called the 
configuration curve. In this paper we establish an upper bound for the degree 
of this curve, assuming the hexapod is general enough. Moreover, 
we provide a construction of hexapods with curves of maximal degree, which  
is based on liaison, a technique used in the theory of algebraic curves.
\end{abstract}

\maketitle

\section*{Introduction}

This paper is devoted to the study of mechanical devices called 
\emph{hexapods}, which are also known as \emph{Stewart Gough platforms}. As described 
in~\cite{Nawratil2014}, the geometry of this kind of mechanical manipulators is 
defined by the coordinates of the~$6$ platform points $p_1, \dotsc, p_6 
\in \R^3$ and of the~$6$ base points $P_1, \dotsc, P_6 \in \R^3$ in one 
of their possible configurations. 
A hexapod is called \emph{planar} if both base points and platform points are 
coplanar, otherwise \emph{non-planar}. 

All pairs of points $(p_i, P_i)$ are connected 
by a rigid body, called \emph{leg}, so that for all possible configurations the 
distance $d_i = \left\| p_i - P_i \right\|$ is preserved (see 
Fig.~\ref{figure:hexapod}). We say that a hexapod 
is \emph{movable}, or admits a \emph{self-motion} if, once we fix the position 
of the base points $\{ P_i \}$, the platform points $\{ p_i \}$ are allowed to 
move in an (at least) one-dimensional set of configurations respecting the 
constraints given by the legs. In this case, each~$p_i$ moves on the sphere 
with center~$P_i$ and radius~$d_i$.

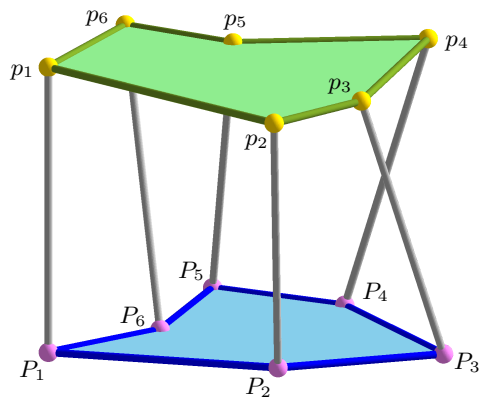
\begin{figure}[!ht]
	\begin{center}
	\begin{overpic}[width=55mm]{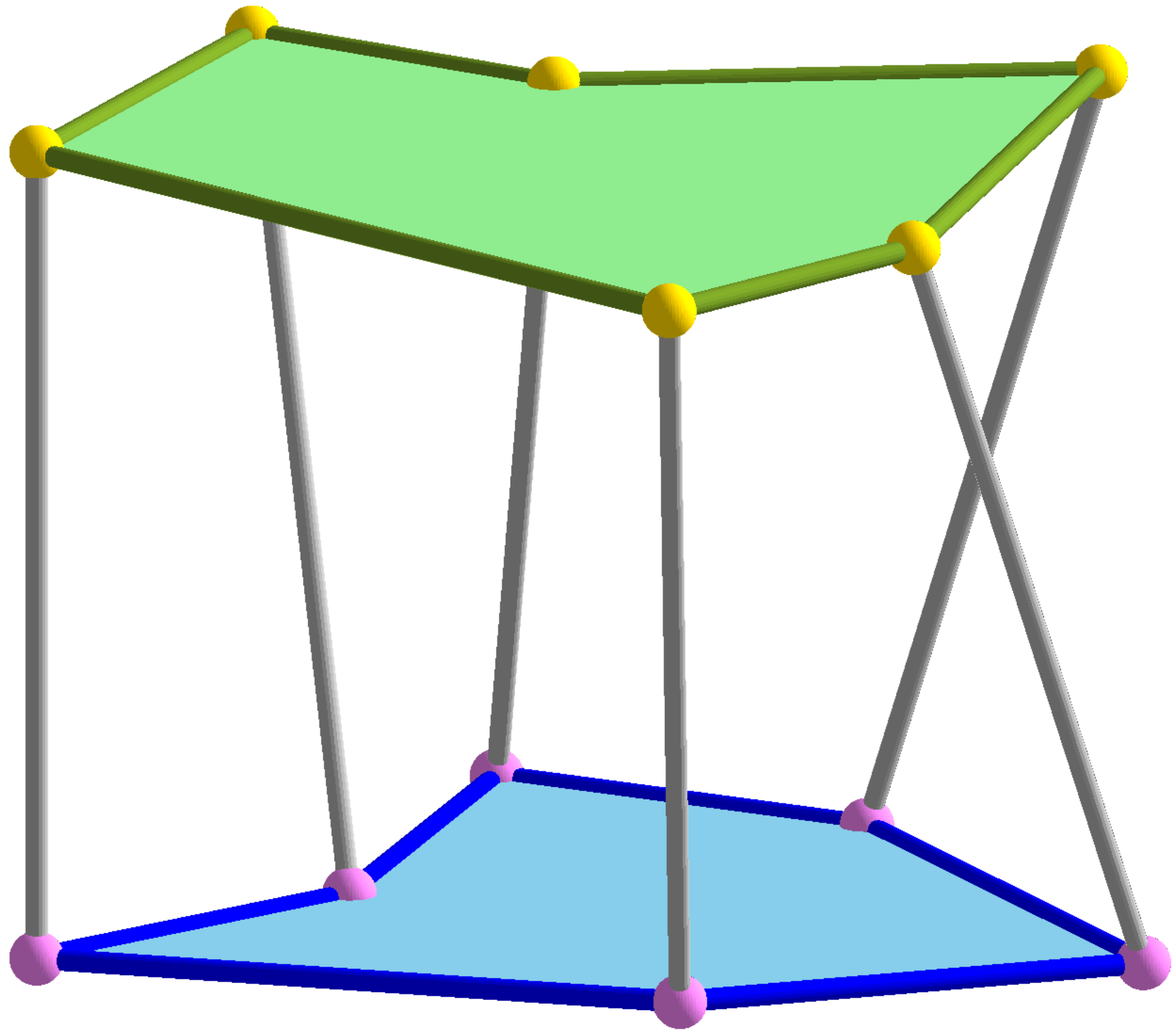}
		\begin{small}
			\put(-4,0){$P_1$}
			\put(50,-4){$P_2$}
			\put(100,2){$P_3$}
			\put(78,19){$P_4$}
			\put(34,24){$P_5$}
			\put(20,13){$P_6$}
			\put(-6,73){$p_1$}
			\put(50,56){$p_2$}
			\put(70,69.5){$p_3$}
			\put(98,80){$p_4$}
			\put(45,85){$p_5$}
			\put(12,86){$p_6$}
		\end{small}
	\end{overpic}
	\end{center}
	\caption{An example of a planar hexapod.}
	\label{figure:hexapod}
\end{figure}

One can associate to each movable hexapod an algebraic curve (if the hexapod 
has mobility one; otherwise it is an algebraic variety of higher dimension) 
contained in the algebraic group~$\SE$ of direct isometries of~$\R^3$; it 
parametrizes the set of possible positions of the platform and because of this 
we call it the \emph{configuration curve} (or the \emph{configuration set} in 
case the mobility is larger than one). Then one may define various invariants 
of a hexapod, such as the degree or the genus of its configuration curve.
Clearly the degree depends on the embedding of the algebraic group~$\SE$ in 
projective space. We mainly use the so-called conformal embedding $\SE \subseteq 
\p^{16}_{\C}$ described in~\cite{GalletNawratilSchicho1, Selig2013} because it 
is most practical for the study of hexapods; in this way we define the {\em 
conformal degree} of a hexapod. In the literature (see for 
example~\cite{HustyKarger}) it is more common to project the configuration curve 
to~$\SO$ and use the well-known embedding $\SO \subseteq \p^3_{\C}$ determined 
by quaternions. In this case we call the degree of the configuration curve the 
\emph{Euler degree}, since the coordinates of~$\p^3_{\C}$ are called \emph{Euler 
parameters}.

The classification of hexapods with mobility one is still an open problem. There is a family
of planar hexapods, discovered by Duporq (see~\cite{duporcq, 
Nawratil_duporcq}), with a configuration curve of conformal degree~$40$, Euler 
degree~$20$, and genus~$41$ (see Section~\ref{degree}); these are the largest 
possible degrees and genus. The family we introduce in this paper is non-planar 
and has conformal degree~$28$ and Euler degree~$14$. We show that this is the 
maximal degree among non-planar hexapods if we exclude some degenerate cases.  
By the way, in this case the genus of the configuration curve is~$23$, 
but we do not know whether this is maximal or not. 

Our linkages are constructed in the following way. We start with~$6$
points $\vec{P} = (P_1, \dotsc, P_6)$ in~$\R^3$, forming the base.
Then we employ basic facts from liaison theory --- a method coming from 
algebraic geometry --- to construct a system of equations for the $6$ 
points $\vec{p} = (p_1, \ldots, p_6)$ forming the platform. 
We conjecture that this system admits solutions for a general choice 
of~$\vec{P}$. Once $\vec{p}$ is fixed (up to scaling), we derive linear 
equations for a scalar~$\gamma$ and a vector $\vec{d} = (d_1, \dotsc, d_6)$ 
implying that the hexapod $(\vec{P}, \gamma \tth \vec{p}, \vec{d})$ has 
mobility one, with a configuration curve of degree~$28$. We 
conjecture that this system of equations has a unique solution for~$\gamma$ 
and a three-dimensional set of solutions for the vector~$\vec{d}$. 
We call a hexapod constructed by this procedure a \emph{liaison hexapod}.
The nature of the two conjectures is so that if they are false, then they are 
falsified by a general choice of base points. We tested the conjectures 
against many random choices, so there is quite a strong experimental evidence 
in their favor. In addition, we found two particular subfamilies (see 
Proposition~\ref{proposition:family_lines} and 
Proposition~\ref{proposition:family_order3}) for which we can prove that the 
properties predicted by the conjectures hold.

The paper is organized as follows. In Section~\ref{degree}, we prove our main 
theorem (Theorem~\ref{thm:bound_degree}) concerning the degree bound for 
configuration curves of general non-planar hexapods of mobility one; this is 
achieved by bounding the degree by the number of intersections of two algebraic 
curves, and then by using some facts from algebraic geometry to control such 
number. The remaining sections deal with the construction of examples. 
In Section~\ref{liaison_hexapods}, we describe liaison hexapods: given a 
general $6$-tuple~$\vec{P}$, we construct --- up to a scaling factor --- a 
candidate platform~$\vec{p}$ applying liaison theory and M\"{o}bius 
photogrammetry, a technique used by the authors in~\cite{GalletNawratilSchicho2} 
to establish necessary conditions for the mobility of pentapods. After that we 
determine the right scaling factor for the platform points and the leg lengths 
so that the hexapod we obtain is movable; both these tasks are achieved by 
inspection of tangency conditions for the configuration curve of the hexapod. 
In Section~\ref{study}, we prove with the aid of symbolic computation 
that the properties predicated by our conjectures hold in the case of two 
particular subfamilies of hexapods, and we exhibit an example of a liaison 
hexapod.

\section{Conformal degree of hexapods}
\label{degree}

The goal of this section is to associate to each hexapod~$\Pi$ a projective 
curve~$K_{\Pi}$, called the \emph{configuration curve} of~$\Pi$. Its degree, 
called the \emph{conformal degree} of~$\Pi$, is an invariant of the hexapod, 
and in this paper we are interested in understanding what are its possible 
maximal values. In particular we report well-known examples of hexapods 
attaining high values for the conformal degree, and eventually we prove a bound 
for the degree of a hexapod satisfying some non-degeneracy conditions 
(Theorem~\ref{thm:bound_degree}).

\subsection{Configuration set and bonds of a hexapod}
\label{degree:configuration}

We recall some definitions from~\cite{GalletNawratilSchicho1} that we need in 
our discussion. Following many other authors, we describe the set of 
admissible configurations of a given hexapod~$\Pi$ by prescribing the set of 
direct isometries (from~$\R^3$ to itself) mapping the initial position of the 
platform to an admissible one. By this we mean that we consider the 
base~$\vec{P}$ as fixed, and if $\vec{p}$ is the initial position of the 
platform points, we look for all $\sigma \in \SE$ (where $\SE$ denotes the group 
of direct isometries) such that
\begin{equation}
\label{equation:spherical_condition}
  \left\| \sigma(p_i) - P_i \right\| \; = \; d_i \quad 
\mathrm{for\ all\ } i \in \{ 1, \dotsc, 6 \},
\end{equation}
where $d_i$ are the given leg lengths of the hexapod~$\Pi$. The constraints 
imposed on isometries $\sigma \in \SE$ by 
Eq.~\eqref{equation:spherical_condition} are called \emph{spherical 
conditions}. 

As shown in~\cite[Section~2, Subsection~2.1]{GalletNawratilSchicho1}, it is 
possible to construct a projective compactification in~$\p^{16}_{\C}$ for 
(the complexification of) the group~$\SE$; we denote it by~$X$. It 
turns out that~$X$ is a projective variety of dimension~$6$ and degree~$40$. We 
call the map $\SE \hookrightarrow \p^{16}_{\C}$ the \emph{conformal 
embedding}\footnote{This name was suggested in a private communication with Jon 
Selig.} of~$\SE$. If we write the spherical condition in the coordinates 
of~$\p^{16}_{\C}$, we obtain six linear conditions, namely we 
determine a linear subspace $H_{\Pi} \subseteq \p^{16}_{\C}$ of codimension~$6$. 
The intersection $K_{\Pi} = X \cap H_{\Pi}$ is defined to be the \emph{complex 
configuration set} of the hexapod~$\Pi$. 

\begin{definition}
	Let $\Pi$ be a hexapod. The \emph{mobility} of~$\Pi$ is defined to be 
the dimension of $K_{\Pi} \cap \mathrm{SE}_{3,\C}$, where $\mathrm{SE}_{3,\C}$ 
is the complexification of~$\SE$ embedded in~$X$. For a hexapod of mobility one 
we define the \emph{conformal degree} to be the degree of the projective 
curve~$K_{\Pi}$.
\end{definition}

\begin{remark}
  From the fact that $X$ is a variety of degree~$40$ and $K_{\Pi}$ is always a 
linear section of~$X$ we immediately see that the conformal degree is always 
smaller than or equal to~$40$.
\end{remark}

One can also attach another degree to a hexapod, by considering only the 
rotational part of the isometries determining its configuration. This is 
classically done by embedding the (complexification of the) group of 
rotations~$\SO$ as an open subset of~$\p^3_{\C}$ using the quaternionic 
description of rotations. In this way we get for a mobility one hexapod~$\Pi$ a 
curve in~$\p^3_{\C}$, and its degree is called the \emph{Euler degree} of~$\Pi$. 
Note that we obtain a point in~$\p^3_{\C}$ if the corresponding 
self-motion is a pure translation. In this case we set the Euler degree to zero.

If $\rho \colon \SE \longrightarrow \SO$ is the map sending a direct isometry 
to its rotational part, by analyzing the definition of the conformal embedding 
one sees that there exists a linear projection $\p^{16}_{\C} \dashrightarrow 
\p^{9}_{\C}$ such that the following diagram is commutative:
\begin{equation}
\label{diagram:compactifications}
  \begin{array}{c}
  \xymatrix{ \SE \ar[rr] \ar[d]_{\rho} & & X \ar@{}[r]|{\subseteq} 
\ar@{-->}[d] & \p^{16}_{\C} \ar@{-->}[d] \\ \SO \ar[r] & \p^3_{\C} 
\ar[r]^-{v_{3,2}} & V_{3,2} \ar@{}[r]|{\subseteq} & \p^9_{\C} } 
  \end{array}
\end{equation}
where $v_{3,2}$ is the Veronese embedding of~$\p^{3}_{\C}$ by quadrics and 
$V_{3,2}$ is its image in~$\p^9_{\C}$.

\smallskip
If we have a closer look at~$X$, we discover that it can be written as the 
disjoint union $\mathrm{SE}_{3,\C} \cup B$ --- where 
$B$ is a special hyperplane section of~$X$, 
called the \emph{boundary}. The boundary can be, in turn, decomposed into~$5$
subsets:
\begin{description}
\item[vertex] the only real point in~$B$, a singular point with 
multiplicity~$20$; it is never contained in the complex configuration set of a 
hexapod;
\item[collinearity points] if $K_{\Pi}$ contains such a point, then either the
platform points or the base points are collinear;
\item[similarity points] if $K_{\Pi}$ contains such a point, then there are normal 
projections of platform and base to a plane such that the images are similar;
\item[inversion points] if $K_{\Pi}$ contains such a point, then there are normal
projections of platform and base to a plane such that the images are related
by an inversion;
\item[butterfly points] if $K_{\Pi}$ contains such a point, then there are two lines,
one in the base and one in the platform, such that any leg has either a base 
point in the base line or a platform point in the platform line.
\end{description}

\begin{remark}
\label{remark:center_projection}
	The center of the projection $\p^{16}_{\C} \dashrightarrow \p^9_{\C}$ 
sending~$X$ to the Veronese variety~$V_{3,2}$ is the linear space spanned by 
similarity points, which contains also the collinearity points and the vertex.
\end{remark}

Inversion points form an open subset of~$B$, while all other sets of points form 
proper (quasi-projective) subvarieties of~$B$. The points of the configuration 
set~$K_{\Pi}$ of a hexapod~$\Pi$ that lie on the boundary~$B$ are called the 
\emph{bonds} of~$\Pi$. As we see from the previous description, although bonds 
do not represent configurations, their presence constrains the geometry of the 
corresponding hexapod. 

\smallskip
We now discuss a few well-known special cases of hexapods attaining high values 
for the conformal and the Euler degree, mainly in order to exclude them later, 
when stating and proving our main theorem (Theorem~\ref{thm:bound_degree}).

\subsubsection{Planar hexapods}

For a general planar pentapod, it is possible to add an additional 
leg without changing the configuration set (see~\cite{duporcq, 
Nawratil_duporcq}), hence obtaining a movable hexapod. Its configuration curve 
is the intersection of~$X$ with a linear subspace of codimension~$5$, hence it 
has conformal degree~$40$, and its Euler degree is~$20$. The genus can easily 
be computed by the	Hilbert series of~$X$, and it turns out to 
be~$41$\footnote{Moreover, the Hilbert series shows 
that~$K_{\Pi}$ is half-canonical, namely embedded by a linear 
system~$\mathscr{L}$ such that $2\mathscr{L}$ is canonical.}.

\subsubsection{Non-planar equiform hexapods}

A hexapod is called \emph{equiform} if there is a similarity $\R^3 
\longrightarrow \R^3$ mapping base points to platform points. Non-planar 
equiform hexapods are discussed in~\cite{nawratilequiform}.

Equiform hexapods of mobility one admit configuration curves of 
degree~$38$. The assumption on the existence of the similarity implies that for 
every direction in~$\R^3$, the projections of both base and platform points 
along that direction are similar, hence every direction in~$\R^3$ 
determines a similarity point in~$K_{\Pi}$. Therefore the intersection 
of~$K_{\Pi}$ with the boundary~$B$ is a curve, and one can show that it is 
actually a conic. Thus the configuration curve of such a hexapod has two 
components, one given by the Zariski closure of $K_{\Pi} \cap 
\mathrm{SE}_{3,\C}$, and the other by the intersection of $K_{\Pi}$ with~$B$.

\subsubsection{General case}

If we assume that $\Pi$ is neither planar nor equiform, then $K_\Pi$ 
intersects~$B$ in at most a finite number of points 
(see~\cite{GalletNawratilSchicho1}). For this case, we show a useful 
proposition relating the conformal degree and the Euler degree.

\begin{proposition} 
\label{proposition:euler_conformal}
Let $\Pi$ be a mobility one hexapod, neither planar nor equiform, then its Euler 
degree is at most half its conformal degree.
\end{proposition}
\begin{proof}
Consider the linear projection $\p^{16}_{\C} \dashrightarrow \p^9_{\C}$ from 
Diagram~\eqref{diagram:compactifications}. Then the configuration 
curve~$K_{\Pi}$ is mapped by such projection to a curve~$K'_{\Pi}$ in~$V_{3,2}$ 
such that the preimage of~$K'_{\Pi}$ under the Veronese embedding~$v_{3,2}$ 
determines the Euler degree of~$\Pi$. 

Let $r$ be the number of intersection points, counted with multiplicity, 
of~$K_\Pi$ with the center of the projection. Since by 
Remark~\ref{remark:center_projection} this is the cardinality of similarity 
bonds plus the cardinality of collinearity bonds of~$\Pi$, the number~$r$ is 
finite. By a well-known formula for the degree of a projection, the degree of 
the image of~$K_\Pi$ is a factor of~$d-r$, where~$d$ is the conformal degree, 
with equality if and only if the restriction of the projection to~$K_\Pi$ is 
birational to the image. On the other hand, the Veronese map just doubles the 
degree, hence the statement follows.
\end{proof}

We aim at bounding the conformal degree of a hexapod~$\Pi$ that satisfies 
some genericity assumptions (for example, we will assume $\Pi$ to be non-planar 
and non-equiform), and we will determine it by means of two estimates: first of 
all we connect the conformal degree to the intersection number of two algebraic curves 
defined by base and platform points only, and then the latter will be 
bounded using some arguments from algebraic geometry. To do so we need to 
introduce a new tool, called \emph{M\"{o}bius photogrammetry}, and this is the 
content of the next subsection.

\subsection{M\"obius photogrammetry for hexapods}
\label{degree:photogrammetry}

Notice that a necessary condition for the mobility of a hexapod is the 
existence of bonds of some type: in fact if a hexapod~$\Pi$ is movable, then 
its configuration set~$K_{\Pi}$ intersects the boundary~$B$ non-trivially. In 
particular, when the mobility is one, the bigger the degree of~$K_{\Pi}$, the 
bigger the number of bonds --- counted with multiplicity.

If we assume that $\Pi$ is neither planar nor equiform nor does it satisfy a 
butterfly condition, but has mobility one, then by the characterization of 
boundary points we get that there exist planar projections of base and platform 
such that their images are related by either a similarity or an inversion. 
In the second case, we can replace the inversion also by the composition of the 
inversion with a reflection, by considering the projection from the opposite 
direction. Then, in both cases, we have projections of base and platform 
points that are M\"{o}bius equivalent, i.e.\ equivalent up to M\"{o}bius 
transformations. 

Hence, if we are interested in hexapods with high conformal degree, then we 
should look for hexapods with a high number of bonds. Since the presence of 
inversion or similarity bonds corresponds to M\"obius equivalent orthogonal 
projections of base and platform, we are going to study such projections.

\subsubsection{Construction of the photographic map}
\label{degree:photogrammetry:construction}

In~\cite{GalletNawratilSchicho2} the authors introduced the concept of
\emph{M\"obius photogrammetry} for $5$-tuples of points in~$\R^3$, and this 
technique allowed to deduce some results about pentapods with mobility two. The idea 
behind it is to take ``pictures'' of a configuration of~$5$ points 
in~$\R^3$, namely consider its orthogonal projection onto~$\R^2$ along all 
possible directions, and try to deduce properties of this configuration in~$\R^3$ from the 
knowledge of its projections. We briefly recall the construction of the 
photographic map, adapting it to our setting, namely considering $6$-tuples of 
distinct points --- for definitions and proofs we refer 
to~\cite{GalletNawratilSchicho2}. 

We start with a vector~$\vec{A}= (A_1, \dotsc, A_6)$ of~$6$ distinct points in~$\R^3$. For each 
direction $\varepsilon \in S^2$ in~$\R^3$, where $S^2$ denotes the unit sphere, 
we consider an orthogonal projection $\pi_{\varepsilon} \colon \R^3 
\longrightarrow \R^2$ along~$\varepsilon$. Hence for each $\varepsilon \in 
S^2$ we have a tuple $\pi_{\varepsilon}(\vec{A}) = \bigl( \pi_{\varepsilon} 
(A_1), \dotsc, \pi_{\varepsilon}(A_6) \bigr)$ of projected points of~$\vec{A}$, 
namely we have a $6$-tuple of elements of~$\R^2$. By identifying~$\R^2$ 
with~$\C$ and embedding the latter in~$\p^1_{\C}$, we can think 
of~$\pi_{\varepsilon}(\vec{A})$ as of an element of~$\left(\p^1_{\C}\right)^6$.

Since M\"obius transformations form the automorphism group of~$\p^1_{\C}$,  
the following two concepts are equivalent:
\begin{itemize}
	\item[$\cdot$]
		a $6$-tuple $\pi_{\varepsilon}(\vec{A})$ in $(\R^2)^6$, considered up to 
M\"obius transformations;
	\item[$\cdot$]
		a $6$-tuple $\pi_{\varepsilon}(\vec{A})$ in $\left(\p^1_{\C}\right)^6$, 
considered up to automorphism of $\p^1_{\C}$.
\end{itemize}

Let $M_6$ denote the moduli space of~$6$ points in~$\p^1_{\C}$, and 
denote by $\varphi \colon \left(\p^1_{\C}\right)^6 \!\! \dashrightarrow 
M_6$ the corresponding quotient map, namely the function sending a 
configuration of $6$ points to its class modulo automorphisms of~$\p^1_{\C}$. 
The previous notation encodes the fact that $\varphi$ is not defined on 
the whole~$\left( \p^1_{\C} \right)^6$, since we know from Geometric 
Invariant Theory that we need to restrict to the open set~$\mcal{U}$ of~$\left( 
\p^1_{\C} \right)^6$ of $6$-tuples where no four points coincide in order to 
build a well-behaved theory. The previous equivalence allows us to think of a 
$6$-tuple $\pi_{\varepsilon}(\vec{A})$ of points in~$\R^2$, considered up to 
M\"{o}bius transformations, as a point in~$M_6$, in particular as the image 
of~$\pi_{\varepsilon}(\vec{A})$ under the quotient map~$\varphi$. In order
to give an explicit formulation of the quotient map~$\varphi$ we follow, as 
in~\cite{GalletNawratilSchicho2}, the graphical approach provided 
in~\cite{HowardMillisonSnowdenVakil}:

\begin{itemize}
	\item[1.]
		Consider a convex hexagon in the plane and construct all plane
undirected multigraphs without loops whose set of vertices coincides with the
set of vertices of the hexagon and that satisfy the following conditions:	
		\begin{itemize}
			\item[$\cdot$]
				edges are given by segments;
			\item[$\cdot$]
				any two edges do not intersect;
			\item[$\cdot$]
				the valency of every vertex is~$1$.
		\end{itemize}
		There are exactly~$5$ of these graphs, as shown in 
		Fig.~\ref{figure:graphs}.
		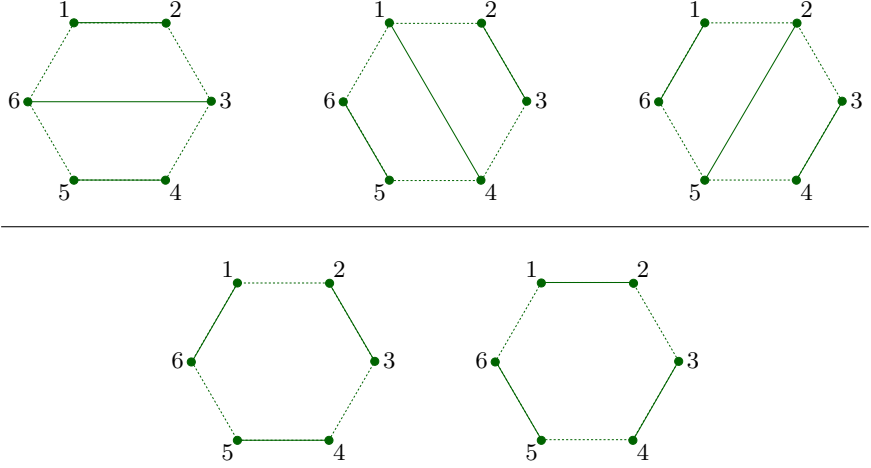
\begin{figure}[ht!]
			\begin{tabular}{ccc}
				\subfigure{
					\begin{overpic}[width=0.2\textwidth]{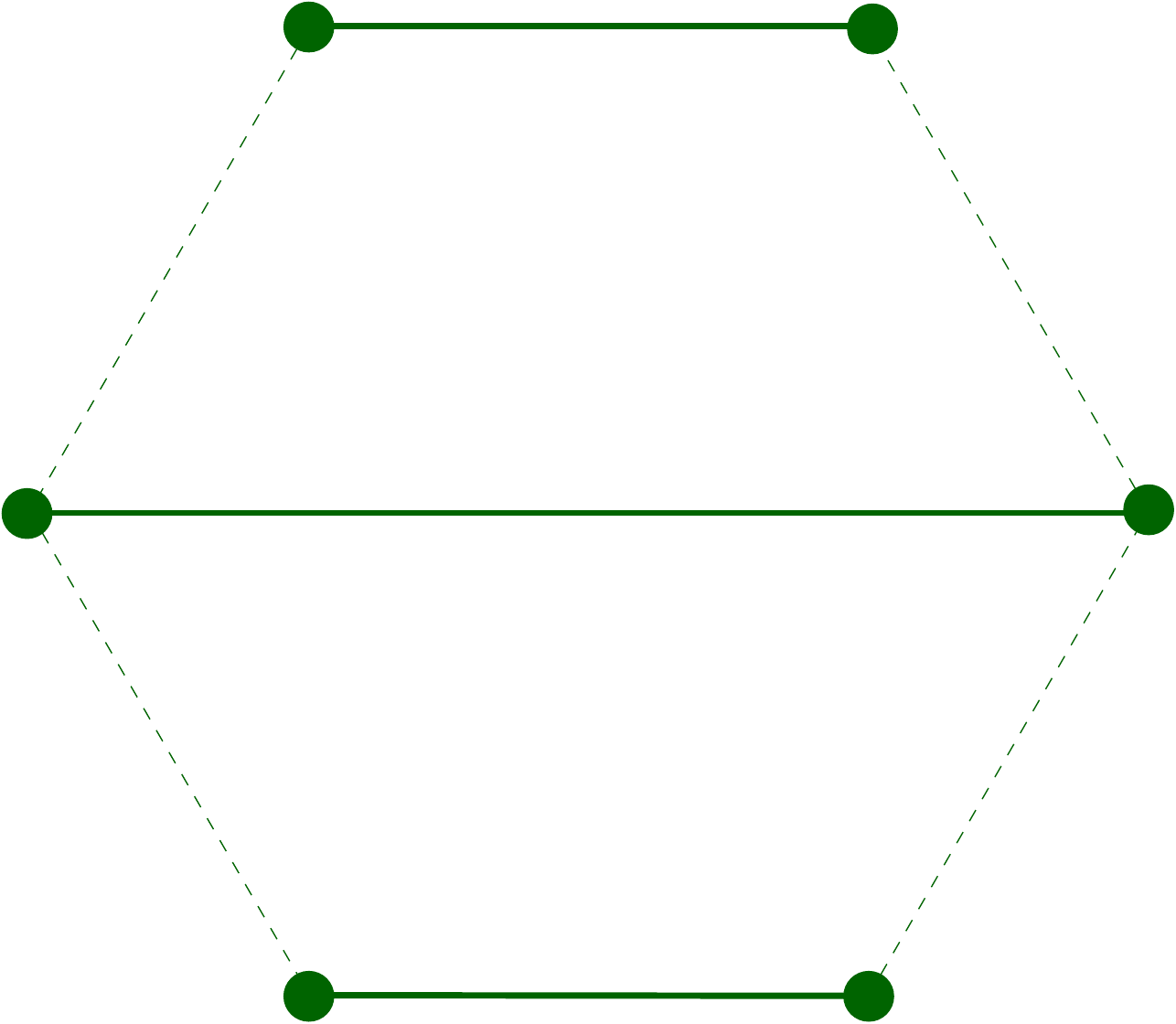}
						\begin{small}
							\put(18,88){$1$}
							\put(76,88){$2$}
							\put(102,40){$3$}
							\put(76,-8){$4$}
							\put(18,-8){$5$}
							\put(-8,40){$6$}
						\end{small} 
					\end{overpic}
				} \hspace{0.4cm} & \hspace{0.4cm}
				\subfigure{
					\begin{overpic}[width=0.2\textwidth]{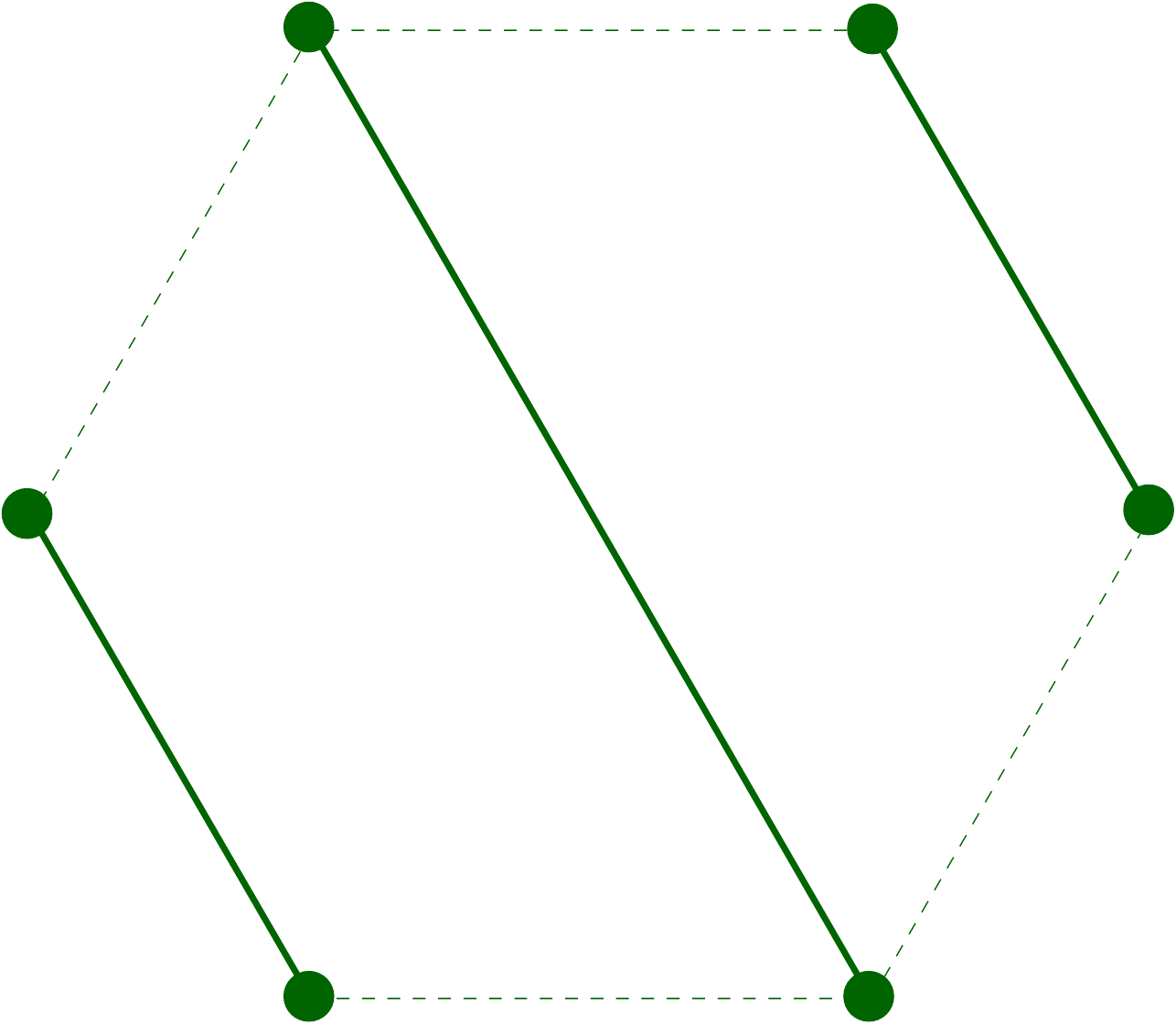}
						\begin{small}
							\put(18,88){$1$}
							\put(76,88){$2$}
							\put(102,40){$3$}
							\put(76,-8){$4$}
							\put(18,-8){$5$}
							\put(-8,40){$6$}
						\end{small} 
					\end{overpic}
				} \hspace{0.4cm} & \hspace{0.4cm}
				\subfigure{
					\begin{overpic}[width=0.2\textwidth]{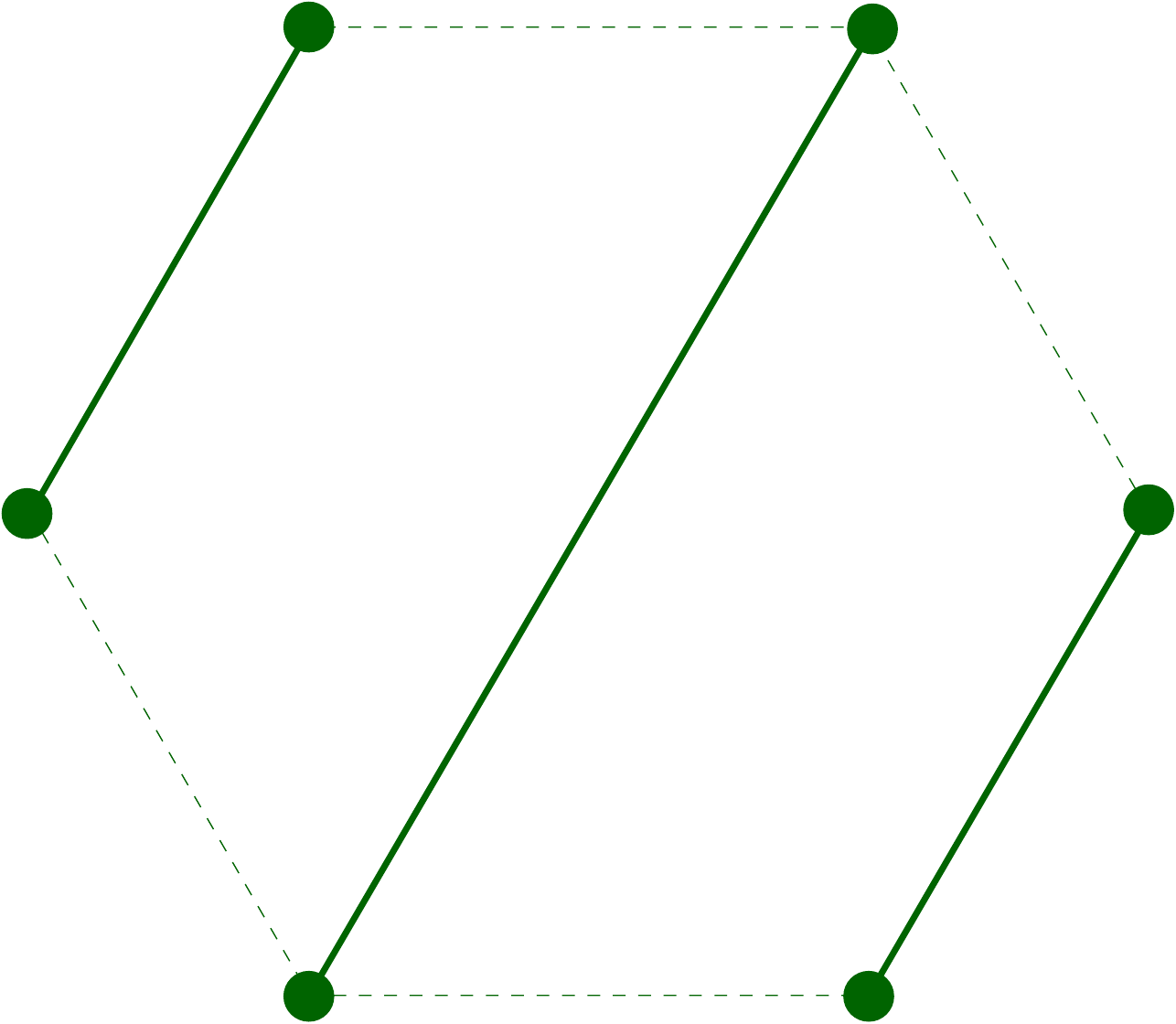}
						\begin{small}
							\put(18,88){$1$}
							\put(76,88){$2$}
							\put(102,40){$3$}
							\put(76,-8){$4$}
							\put(18,-8){$5$}
							\put(-8,40){$6$}
						\end{small}
					\end{overpic}
				} \\ [4mm] \hline & & \\ 
				\multicolumn{3}{c}{
				\subfigure{
					\begin{overpic}[width=0.2\textwidth]{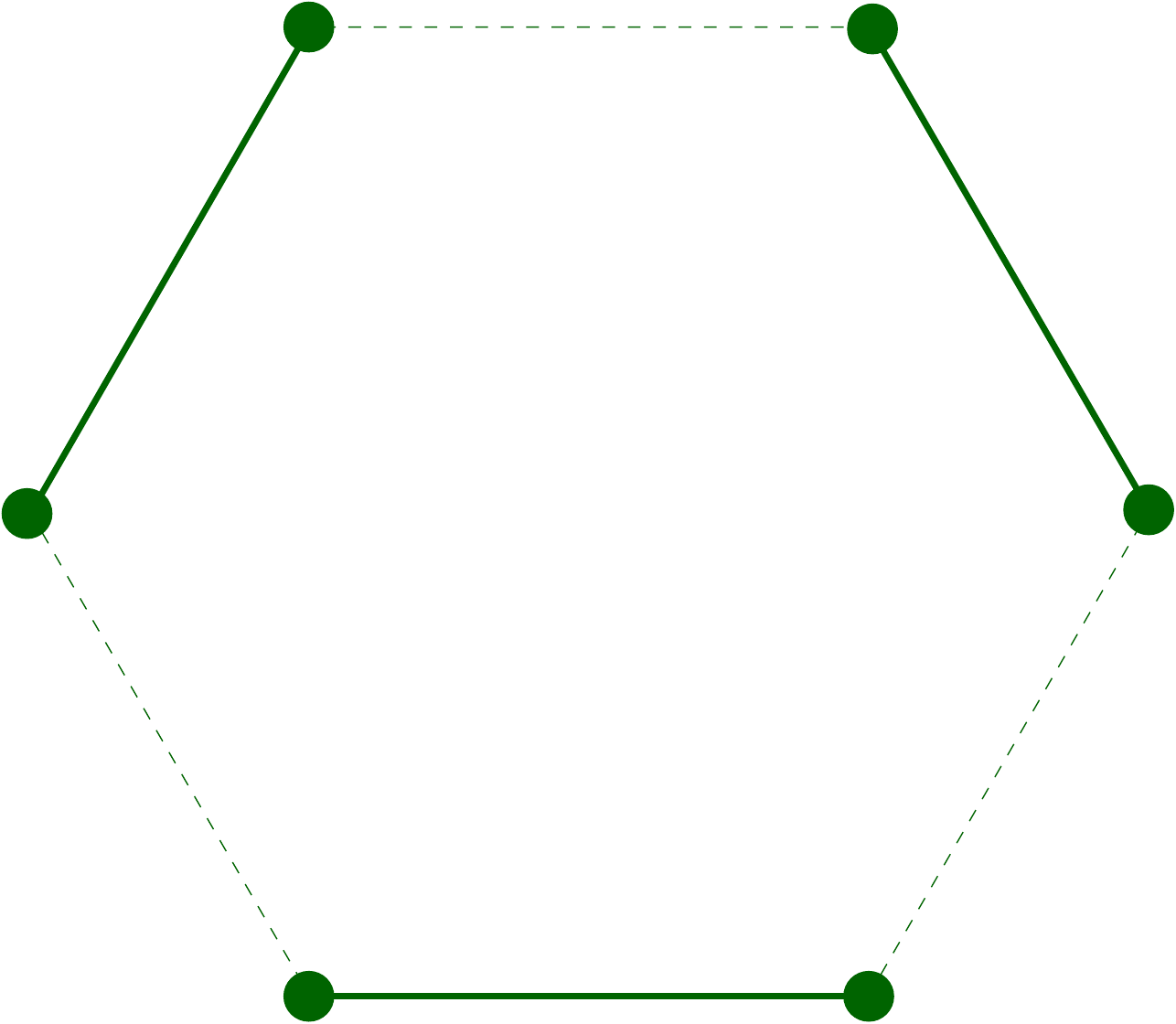}
						\begin{small}
							\put(18,88){$1$}
							\put(76,88){$2$}
							\put(102,40){$3$}
							\put(76,-8){$4$}
							\put(18,-8){$5$}
							\put(-8,40){$6$}
						\end{small} 
					\end{overpic}
				} \hspace{1cm} 
				\subfigure{
					\begin{overpic}[width=0.2\textwidth]{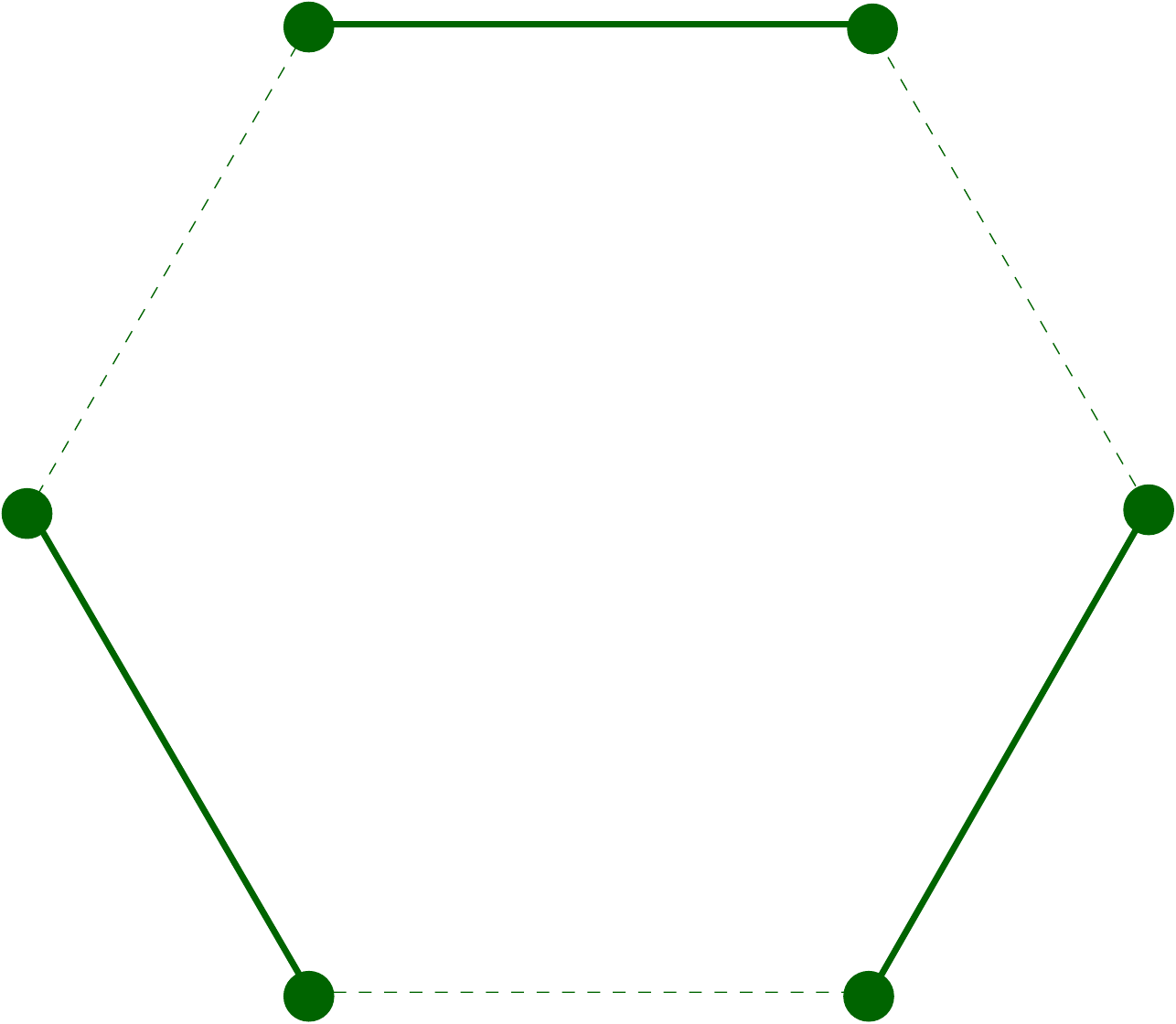}
						\begin{small}
							\put(18,88){$1$}
							\put(76,88){$2$}
							\put(102,40){$3$}
							\put(76,-8){$4$}
							\put(18,-8){$5$}
							\put(-8,40){$6$}
						\end{small}
					\end{overpic}
				} } 
			\end{tabular}
		\caption{The only five planar undirected multigraphs without loops with 
		vertices on a regular hexagon, valency~$1$ and non-intersecting edges.}
		\label{figure:graphs}
		\end{figure}
		\item[2.]
		Associate to each graph~$G$ with set of edges~$E$ a homogeneous polynomial 
		in the coordinates $\{(a_i:b_i)\}$ of $\left( \p^1_{\C} \right)^6$ in the 
		following way:
		\[ \varphi_{G} \; = \; \prod_{ \substack{(i,j) \in E \\ i < j} } (a_i \tth 
		b_j - a_j \tth b_i). \]
		For example the polynomial associated to the first graph in 
		Fig.~\ref{figure:graphs} is
		\[ \varphi_0 = (a_1 \tth b_2 - a_2 \tth b_1) (a_3 \tth b_6 - a_6 \tth b_3)
		(a_4 \tth b_5 - a_5 \tth b_4). \]
	\item[3.]
		These five polynomials determine a rational map $\varphi \colon \left( 
		\p^1_{\C} \right)^6 \!\! \dashrightarrow \p^{4}_{\C}$.
	\item[4.]
		Consider the open set
		\[ \mathcal{U} \; = \; \big\{ (m_1, \dotsc, m_6) \, : \; \textrm{no
		four of the points } m_i \textrm{ coincide} \big\} \; \subseteq \; \left(
		\p^1_{\C} \right)^6. \]
	\item[5.]
		The image of~$\varphi_{|_\mcal{U}}$, namely the restriction 
		of~$\varphi$ to~$\mcal{U}$, gives an embedding of~$M_6$ in~$\p^4_{\C}$.
\end{itemize}

Summing up, once we fix a $6$-tuple~$\vec{A}$ of distinct points in~$\R^3$, we 
can associate to each direction $\varepsilon \in S^2$ the image (under the 
quotient map~$\varphi$) of the projected $6$-tuple~$\pi_{\varepsilon}(\vec{A})$. 
In this way we construct a map from~$S^2$ to~$M_6$. In order to make it a 
morphism between projective varieties, we identify~$S^2$ with the conic
$C = \{ x^2 + y^2 + z^2 = 0 \}$ in~$\p^{2}_{\C}$. We obtain a map $f_{\vec{A}} 
\colon C \longrightarrow M_6$. 

\noindent If we write $A_i = (s_i, t_i, u_i)$ and we set for all $i, j \in \{1, 
\dotsc, 6\}$ with $i < j$
\[ H_{ij} \; = \; (s_i - s_j) \tth x + (t_i - t_j) \tth y + (u_i - u_j) \tth z, 
\]
then the components of~$f_{\vec{A}}$ have the following structure:
\begin{equation}
\label{equation:structure_photographic_map}
	\begin{array}{rcccc}
		\left( f_{\vec{A}} \right)_{0} & = & H_{12} & H_{36} & H_{45} \\
		\left( f_{\vec{A}} \right)_{1} & = & H_{14} & H_{23} & H_{56} \\
		\left( f_{\vec{A}} \right)_{2} & = & H_{16} & H_{25} & H_{34} \\
		\left( f_{\vec{A}} \right)_{3} & = & H_{16} & H_{23} & H_{45} \\
		\left( f_{\vec{A}} \right)_{4} & = & H_{12} & H_{34} & H_{56} \\
	\end{array}
\end{equation}
From this explicit description we see that in particular this map is algebraic.

\begin{definition}
\label{definition:photographic_map}
The regular map $f_{\vec{A}} \colon C \longrightarrow M_6$ we obtain is called 
the \emph{photographic map} of~$\vec{A}$. 
\end{definition}

\begin{definition}
\label{definition:moebius_curve}
 Let $\vec{A}$ be a $6$-tuple of distinct points in~$\R^3$ --- from 
now on we will omit the adjective ``distinct'', but throughout the paper we will 
always consider this situation. The image of the photographic map~$f_{\vec{A}}$ 
is a rational curve in~$M_6$, that we call the \emph{M\"{o}bius curve} 
of~$\vec{A}$.
\end{definition}

\begin{remark}
	As noticed in~\cite{GalletNawratilSchicho2}, the map~$f_{\vec{A}}$ is a 
morphism of real varieties, namely it respects the real structures of~$C$ 
and~$M_6$ induced by the standard ones of~$\p^2_{\C}$ and~$\p^4_{\C}$, 
respectively.
\end{remark}

\begin{fact}
  One can prove that~$M_6$ is embedded in~$\p^4_{\C}$ as a cubic threefold,
called the \emph{Segre cubic primal}, whose equation is
\[ x_0 x_1 (x_0 + x_1 + x_2 + x_3 + x_4) - x_2 x_3 x_4 \; = \; 0. \]
	The threefold~$M_6$ contains exactly~$15$ planes, corresponding to
equivalence classes of configurations $(m_1, \dotsc, m_6)$ where $m_i = m_j$ for
some $i \neq j$. 
	Furthermore, the maximum possible number of nodes for a cubic hypersurface 
in~$\p^4_{\C}$, namely~$10$, is attained by the Segre cubic. The nodes 
correspond to configurations of points $(m_1, \dotsc, m_6)$ where three points 
coincide, and are the only singular points of~$M_6$. There are $\binom{6}{3} = 
20$ such configurations, but the following phenomenon happens in~$M_6$: if we 
partition $\{1,2,3,4,5,6\} = \{i,j,k\} \cup \{ u,v,w \}$, then configurations of 
points for which $m_i = m_j = m_k$ are sent by the quotient map $\varphi \colon 
\left(\p^1_{\C}\right)^6 \!\! \dashrightarrow M_6$ to the same point of~$M_6$ 
as configurations of points for which $m_u = m_v = m_w$.

	The interested reader can find proofs and explanations for the stated facts
in~\cite[Chapter~9, Subsection~9.4.4]{Dolgachev}.
\end{fact}

By construction, whenever a hexapod~$\Pi$ admits an inversion or similarity 
bond, then the M\"{o}bius curves of its base and platform intersect. However, 
due to the construction of the moduli space~$M_6$, this is not the only 
situation when we get intersections between the two curves: suppose in fact 
that $\Pi$ admits three collinear base points, say $P_1, P_2$ and $P_3$, and 
three collinear platform points, say $p_4, p_5$ and $p_6$, so that we have a 
butterfly bond; then the M\"{o}bius curves intersect in a node of~$M_6$, 
since the projections along the lines $\overrightarrow{P_1 P_2 P_3}$ and 
$\overrightarrow{p_4 p_5 p_6}$ determine configurations in~$\R^2$ where 
three points coincide.

\subsubsection{Properties of the photographic map}

The following two lemmata, in analogy with Lemma~3.7 and~3.8 
in~\cite{GalletNawratilSchicho2}, describe the possible behavior of the 
photographic map. The proof technique is similar to the one
in~\cite{GalletNawratilSchicho2}, but we report the proofs for 
self-containedness purposes. 

\begin{lemma}
\label{lemma:6_space}
	Let $\vec{A} = (A_1, \dotsc, A_6)$ be a $6$-tuple of points in~$\R^3$. If the 
$\{A_i\}$ are not coplanar, then the photographic map $f_{\vec{A}} \colon C
\longrightarrow M_6$ is birational to a rational curve of degree~$6$ or~$4$ 
in~$M_6$.
\end{lemma}
\begin{proof}
	We notice that if the lines $\overrightarrow{A_iA_j}$ and
$\overrightarrow{A_kA_l}$ are parallel, then the factors~$H_{ij}$ and~$H_{kl}$ 
in Eq.~\eqref{equation:structure_photographic_map} differ only by a scalar. 
This means that in particular situations, when some of
the~$A_i$ are collinear, the components of~$f_{\vec{A}}$ can have factors in
common that can be simplified, reducing the degree of the map. Hence we
partition the set of possible configurations in two cases, depending on the
number of common factors of the components of~$f_{\vec{A}}$. Notice that, since
we assume that the points are not coplanar, it can never happen that~$5$ or
all~$6$ points are collinear.
\begin{description}
	\item[Case (a)]
		Here no~$4$ points are collinear. In this case the components of
		$f_{\vec{A}}$ do not have any factor in common, so
		\[ \deg{\left( f_{\vec{A}}(C) \right)} \cdot \deg{(f_{\vec{A}})}
		\; = \; 6. \]
		Thus there are only four possibilities: either $f_{\vec{A}}$ is $6:1$ to
		a line, or it is $3:1$ to a conic, or it is $2:1$ to a cubic curve, or 
		$1:1$ to a sextic curve. We prove that the first three situations 
		can never happen. First of all, notice that there are exactly two 
		directions in~$\R^3$ for which the images of~$A_i$ and~$A_j$ under the 
		projection coincide, namely the directions of the oriented 
		lines~$\overrightarrow{A_i A_j}$ and~$\overrightarrow{A_j A_i}$. Denote 
		by~$T_{ij}$ the plane in~$M_6$ of classes of $6$-tuples $(m_1, \dotsc, 
		m_6)$ where $m_i = m_j$ (see the Fact in 
		Subsection~\ref{degree:photogrammetry:construction}). Then
		there are exactly~$2$ points in~$C$ that are mapped to the plane~$T_{ij}$ 
		--- and they are complex conjugate, since complex conjugation in~$C$ 
		corresponds to the antipodal map in the unit sphere~$S^2 \subseteq \R^3$.

		If $f_{\vec{A}}$ is $3:1$ or $6:1$, then those two points have 
		to be branching points of~$f_{\vec{A}}$. Setting $A_i = (s_i, t_i, u_i)$ 
		one can check that
		\[
		  f_{\vec{A}}^{-1}(T_{ij}) \; = \; \bigl\{ (x: y: z) \in C \, : \, 
		  H_{ij}(x,y,z) = 0 \bigr\},
		\]
		where we recall that $H_{ij} = (s_i - s_j) \tth x + (t_i - t_j) \tth y 
		+ (u_i - u_j) \tth z$. If the points in the preimage of~$T_{ij}$ are 
		branching points, then the line $\{ H_{ij} = 0 \} \subseteq \p^2_{\C}$ 
		should intersect~$C$ tangentially at those points. However, this is 
		impossible, since both~$C$ and the line are real varieties, so if they are 
		tangent their intersection point is real, but $C$ has no real points. In 
		this way we rule out the $3:1$ and the $6:1$ case. 

		Suppose now that 
		$f_{\vec{A}}$ is $2:1$. We are going to show that $\vec{A}$ should be 
		planar, contradicting the hypothesis. In fact, by assumption $f_{\vec{A}}$ 
		factors through an involution of~$C$ --- the one swapping the fibers 
		of~$f_{\vec{A}}$; one can show that such involution is a real automorphism 
		of~$C$, and hence corresponds to a rotation by~$180$ degrees of the unit 
		sphere~$S^2$ along some axis~$\ell$. The fibers $f_{\vec{A}}^{-1}(T_{ij})$ 
		are invariant under the involution, so they are contained in the subset 
		of~$S^2$ invariant under the rotation, namely the intersection $S^2 \cap 
		\ell$ together with the maximal circle in~$S^2$ lying on a plane orthogonal 
		to~$\ell$. A case-by-case analysis proves that all 
		direction~$\overrightarrow{A_i A_j}$ belong to such maximal circle, this 
		implying that~$\vec{A}$ is planar. Hence the birational case is the 
		only possible.
	\item[Case (b)]
		Here exactly $4$ points are collinear.
		In this case the components of~$f_{\vec{A}}$ have one factor in common, 
		leading to
		\[ \deg{\left( f_{\vec{A}}(C) \right)} \cdot \deg{(f_{\vec{A}})}
		\; = \; 4. \]
		We have three possibilities: either $f_{\vec{A}}$ is $4:1$ to
		a line, or it is $2:1$ to a conic, or $1:1$ to a quartic curve. Arguing as 
		in Case~(a) we prove the statement. \qedhere
\end{description}
\end{proof}

\begin{lemma}
\label{lemma:6_plane}
	Let $\vec{A} = (A_1, \dotsc, A_6)$ be a $6$-tuple of points in~$\R^3$. If 
the~$\{A_i\}$ are coplanar, but not collinear, then the photographic map
$f_{\vec{A}} \colon C \longrightarrow M_6$ is $2:1$ to a rational curve of 
degree~$3$, $2$ or~$1$ in~$M_6$.
\end{lemma}
\begin{proof}
  Since $\vec{A}$ is planar, then all directions~$\overrightarrow{A_i A_j}$ 
belong to a maximal circle of the unit sphere~$S^2$. Let $\ell$ be 
the line through the origin orthogonal to the plane spanned by such maximal 
circle. Reversing the argument in the proof of Lemma~\ref{lemma:6_space}, 
Case~(a) one can prove that the map~$f_{\vec{A}}$ factors through the $2:1$ 
involution~$\tau$ of~$C$ 
determined by the rotation of~$S^2$ by $180$ degrees around~$\ell$. Then we 
have $f_{\vec{A}} = g_{\vec{A}} \circ \tau$. From now on the proof works as in 
Lemma~\ref{lemma:6_space}, and the following cases arise:
\begin{description}
	\item[Case (a)]
		Suppose that no $4$ points of~$\vec{A}$ are collinear. Then the components 
		of~$f_{\vec{A}}$ do not have any factor in common, and $g_{\vec{A}}$ is 
		birational to a cubic.
	\item[Case (b)]
		Suppose that $4$ points of~$\vec{A}$ are collinear, but no $5$ points are 
		so. Then the components of~$f_{\vec{A}}$ have exactly one factor in common, 
		and $g_{\vec{A}}$ is birational to a conic.
	\item[Case (c)]
		Suppose that $5$ points of~$\vec{A}$ are collinear. Then the components 
		of~$f_{\vec{A}}$ have exactly two factors in common, and 
		$g_{\vec{A}}$ is birational to a line. \qedhere
\end{description}
\end{proof}

We conclude this section by showing that the photographic map of some
$6$-tuples of points extends to a morphism defined on the whole 
plane~$\p^2_{\C}$. From this we deduce geometric constraints for 
M\"{o}bius curves. We hope that this can be the first step towards a 
complete geometric characterization of M\"{o}bius curves as curves in~$M_6$. 
We focus on the case when M\"{o}bius curves are of degree~$6$, which is the most 
interesting for our application.

\begin{lemma}
\label{lemma:moebius_extend}
	Let $\vec{A}$ be a $6$-tuple of points in~$\R^3$ and let $f_{\vec{A}} \colon 
C \longrightarrow M_6$ be its photographic map, where the resulting  M\"{o}bius curve 
is of degree~$6$. Then $f_{\vec{A}}$ extends to a morphism $F_{\vec{A}} \colon 
\p^2_{\C} \longrightarrow M_6$. 
\end{lemma}
\begin{proof}
	We only need to prove that the map~$F_{\vec{A}}$ does not have base points 
in~$\p^2_{\C}$. Recall the structure of the map~$f_{\vec{A}}$ described by 
Eq.~\eqref{equation:structure_photographic_map}: we infer that a base point 
has to vanish on at least one polynomial~$H_{ij}$ for each component 
of~$F_{\vec{A}}$. A direct inspection shows that this would imply that at least 
$4$ points of~$\vec{A}$ are collinear, but this is impossible because by 
Lemma~\ref{lemma:6_space} the curve~$D$ would not have degree~$6$, 
contradicting the hypothesis. 
\end{proof}

\begin{proposition}
\label{proposition:characterization_moebius}
	Let $D$ be a smooth M\"{o}bius curve of degree~$6$. Then
	\begin{itemize}
		\item[$\cdot$]
		$D$ can be defined by real polynomials, but has no real points;
		\item[$\cdot$]
		$D$ is contained in a linear projection, defined by real polynomials, of 
		the third Veronese embedding of~$\p^2_{\C}$.
	\end{itemize}
\end{proposition}
\begin{proof}
	By construction $D$ is a real variety, since $f_{\vec{A}}$ is a real map and 
$C$ is a real variety; since $D$ is smooth and using Lemma~\ref{lemma:6_space} 
we have that $f_{\vec{A}}$ is an isomorphism, hence $D$ has no real points, 
because this holds for~$C$. Eventually from Lemma~\ref{lemma:moebius_extend} we 
get that $D$ is contained in a linear projection, defined by real polynomials, 
of the third Veronese embedding of~$\p^2_{\C}$, which is real by 
construction. One notices that such projections is the complete intersection 
of~$M_6$ with another cubic hypersurface.
\end{proof}

\subsection{An upper bound for the conformal degree}
\label{degree:bound}

We have now all the tools needed to prove the bound on the conformal degree of 
a hexapod (Theorem~\ref{thm:bound_degree}). We split the argument in two parts: 
first we prove that the conformal degree of a hexapod~$\Pi$ is less than or 
equal to twice the number of intersections of the M\"{o}bius curves of the base 
and platform of~$\Pi$ (Theorem~\ref{theorem:bound_degree}); then we 
show that such intersection cannot be composed by more than~$14$ points 
(Proposition~\ref{proposition:bound_intersection_moebius}). For our proofs to 
work we need to exclude some degenerate cases --- planarity, equiformity and a 
certain kind of collinearity.

The reader should be warned about the length of these two proofs.

\begin{theorem}
\label{theorem:bound_degree}
  Let $\Pi$ be a mobility one hexapod. Let $D_1 = f_{\vec{P}}(C)$ and $D_2 = 
f_{\vec{p}}(C)$ be the M\"{o}bius curves of its base and platform. Suppose that 
$\Pi$ is non-planar, non-equiform and no $4$ base or platform points are on a 
line. Then 
\[
  \deg K_{\Pi} \; \leq \; 2 \tth \deg (D_1 \cap D_2),
\]
where the intersection $D_1 \cap D_2$ is meant scheme-theoretically, namely the 
points of intersection are counted with multiplicity.
\end{theorem}
\begin{proof}
  Informally, the proof goes as follows: the degree of~$K_{\Pi}$ is computed by 
intersecting it with a hyperplane, and we choose the hyperplane defining the 
boundary~$B$ of~$X$. This intersection, as we will see, should be counted 
twice, and this means that the degree of~$K_{\Pi}$ is bounded by twice the 
number of bonds of~$\Pi$. Eventually, bonds correspond to common images of the 
photographic maps, so the statement follows.

In order to make the argument precise, we need to rephrase the analysis of 
bonds carried out in~\cite{GalletNawratilSchicho1} in a scheme-theoretical way. 
We start by noticing that the hyperplane~$L$ defining the boundary~$B$ is 
totally tangential to~$X$, namely the scheme-theoretical intersection $X \cap 
L$ has a non-reduced structure, and we have $B = (X \cap L)_{\mathrm{red}}$; 
moreover, as divisors on~$X$, it holds $2B = X \cap L$. All these results can 
be proved via a direct computation using, for example, Gr\"{o}bner bases. 
Therefore, if $\tilde{K}_{\Pi}$ denotes the top-dimensional part of~$K_{\Pi}$ 
--- namely $\tilde{K}_{\Pi}$ is an equidimensional scheme of dimension one --- 
we have that:
\begingroup
\allowdisplaybreaks
\begin{align*}
  \deg K_{\Pi} & = \deg \tilde{K}_{\Pi} = \deg \left( \tilde{K}_{\Pi} \cap L 
\right) \\
               & \leq  \deg \left( K_{\Pi} \cap L \right) \\
               & = \deg \left( H_{\Pi} \cap X \cap L \right) \\
               & \leq 2 \tth \deg \left( H_{\Pi} \cap B \right).
\end{align*}
\endgroup
The last inequality follows from the general fact that if $Y_1$ and $Y_2$ 
are two divisors of~$X$ satisfying $Y_2 = 2 \tth Y_1$, then for any 
subvariety~$Z$ of~$X$ we have $\deg \left(Y_1 \cap Z \right) \leq 2 \tth \deg 
\left( Y_2 \cap Z \right)$. The scheme $B_{\Pi} = 
H_{\Pi} \cap B$ is the scheme of bonds of~$\Pi$, and we are going to prove that 
the number $\deg \left( H_{\Pi} \cap B \right)$ is less than or equal to the 
degree of the 
intersection $D_1 \cap D_2$ of the two M\"{o}bius curves of~$\Pi$. To reach this 
goal, we need to connect the boundary~$B$ of~$X$ to the curve $C \subseteq 
\p^2_{\C}$ used in Definition~\ref{definition:photographic_map} by means of a 
rational map $B 
\dashrightarrow C \times C$. As we already reported in 
Subsection~\ref{degree:configuration}, to every inversion, similarity and 
butterfly point we can associate two directions in~$\R^3$, and since we use the 
curve~$C$ as a model for the unit sphere~$S^2 \subseteq \R^3$, this means that 
to every such point we can associate two elements $v, w \in C$. Notice that our 
assumption on the hexapod~$\Pi$ rules out the existence of collinearity bonds, 
and so it is harmless to exclude them from the picture. Using the notation 
introduced in~\cite[Sections~2.1 and~2.3]{GalletNawratilSchicho1}, we can write
\[
  B \; = \; 
  \left\{ 
  (h: M: x: y: r) \, : \,
  \begin{array}{ll}
    h = 0, & M M^T = M^T M = 0 \\
    M^T y = M x = 0, & \left\langle x, x \right\rangle = \left\langle y, y
\right\rangle = 0
  \end{array}
  \right\}.
\]
If we define
\[
  B' \; = \; B \setminus \Bigl( \{ \text{collinearity points} \} \cup \{ 
\text{vertex} \} \Bigr),
\]
then we get
\[
  B' \; = \; \Bigl\{ \beta \in B \, : \, M \neq 0 \text{ or } \left( x \neq 0 
\text{ and } y \neq 0 \right) \Bigr\}.
\]
If $\beta \in B'$, then there exist $v, w \in C$ and $\lambda, \mu, \alpha \in 
\C$ such that
\begin{equation}
\label{equation:parametrization_bonds}
  M = \alpha \tth v w^T, \quad x = \mu \tth w, \quad y = \lambda \tth x,
\end{equation}
and from the definition of~$B'$ we get that either $\alpha \neq 0$ or $\lambda 
\mu \neq 0$. We define the algebraic map
\[
  \begin{array}{rccc}
    \delta \colon & B' & \longrightarrow & C \times C \\
    & (0: M: x: y: r) & \mapsto & (v,w)
  \end{array}
\]
Notice that the fiber of~$\delta$ over a point~$(v,w)$ is parametrized by 
\[
  \delta^{-1}(v,w) \; \cong \; \bigl\{ \vec{\omega} = (\alpha: \lambda: \mu: r) 
\in \p^3_{\C} \, : \, \alpha \neq 0 \text{ or } \lambda \mu \neq 0 \bigr\}.
\]
This allows to conclude that $B'$ forms an open subvariety of a 
$\p^3_{\C}$-bundle over~$C \times C$.
As remarked before, our hypotheses imply that the scheme of bonds $B_{\Pi} = B 
\cap H_{\Pi}$ is a closed subscheme of~$B'$.

\noindent {\bfseries Claim.} {\it The map $\delta_{|_{B_{\Pi}}}$ is an 
isomorphism.}

\noindent In order to prove the claim, we need to rephrase the so-called 
\emph{pseudo spherical condition}, introduced 
in~\cite[Definition~3.6]{GalletNawratilSchicho1}. The pseudo spherical 
condition imposed by a pair $(P_i,p_i)$ on a bond $(0: M: x :y: r)$ reads as
\[
  r - 2 \left\langle Mp_i, P_i \right\rangle - 2 \tth \left\langle P_i, y 
\right\rangle - 2 \tth \left\langle p_i, x \right\rangle \; = \; 0.
\]
Using Eq.~\eqref{equation:parametrization_bonds}, and setting $W_i = P_i^T w$ 
and $V_i = p_i^T v$ yields
\[
  r - 2 \alpha W_iV_i - 2 \mu W_i - 2 \lambda V_i \; = \; 0.
\]
The scheme~$B_{\Pi}$ is cut out by these $6$ pseudo spherical conditions 
for $i=1,\ldots ,6$. Hence, if we 
define the $4 \times 6$ matrix
\[
  N_{\Pi}(v,w) \; = \;
  \begin{pmatrix}
    W_1 \tth V_1 & W_1 & V_1 & 1 \\
    \vdots & \vdots & \vdots & \vdots \\
    W_6 \tth V_6 & W_6 & V_6 & 1
  \end{pmatrix},
\]
then the scheme~$B_{\Pi}$ is locally defined by
\[
  \Bigl\{ (v, w, \vec{\omega}) \in \mcal{V} \times \mcal{V} \times \mcal{W} \, 
: \, N_{\Pi}(v,w) \cdot \vec{\omega} = 0 \Bigr\},
\]
where $\mcal{V}$ and $\mcal{W}$ are suitable open subvarieties of~$C$ 
and~$\p^3_{\C}$ respectively, and $\vec{\omega} = (\alpha: \lambda: \mu: r)$. 
We restate our previous claim.

\noindent {\bfseries Claim.} {\it The map~$\delta$ maps $B_{\Pi}$ 
isomorphically to the scheme in $C \times C$ cut out by the $4 \times 4$ minors 
of~$N_{\Pi}$.}

\noindent We prove that the rank of the matrix~$N_{\Pi}(v,w)$ is always at 
least~$3$. In fact, if the rank is~$1$ then the collinearity hypothesis is 
violated (all $W_i$ would be equal, and the same for the $V_i$); if the rank 
is~$2$, then the planarity condition is violated (this can be deduced by a 
direct computation, imposing that the second column of~$N_{\Pi}$ is a linear 
combination of the third and the fourth). Hence the rank of~$N_{\Pi}(v,w)$ is 
greater than of equal to~$3$. Notice that column operations 
on~$N_{\Pi}$ correspond to projective transformations in the fibers of~$\delta$, 
while row operations on~$N_{\Pi}$ correspond to the choice of a different system 
of generators for the ideal of~$B_{\Pi}$. Hence we can reduce to the case when 
$N_{\Pi}$ has the form
\[
  N_{\Pi}(v,w) \; = \;
  \begin{pmatrix}
    1 & 0 & 0 & 0 \\
    0 & 1 & 0 & 0 \\
    0 & 0 & 1 & 0 \\
    0 & 0 & 0 & G_1(v,w) \\
    0 & 0 & 0 & G_2(v,w) \\
    0 & 0 & 0 & G_3(v,w) \\
  \end{pmatrix} ,
\]
where $G_i(v,w)$ are rational functions on $C \times C$. The local 
description of~$B_{\Pi}$ becomes
\[
  \left\{ (v,w,\vec{\omega'}) \, : \, 
  \begin{array}{l}
  \vec{\omega'}=(0:0:0:r') \; \text{ for some } r' \neq 0, \;\; \text{and} \\ 
  G_1(v,w) \tth r' = G_2(v,w) \tth r' = G_3(v,w) \tth r' = 0
  \end{array}
  \right\},
\]
while the zero locus of the $4 \times 4$ minors of~$N_{\Pi}$ is locally given by
\[
  \bigl\{ (v,w) \, : \, G_1(v,w) = G_2(v,w) = G_3(v,w) = 0 \bigr\},
\]
and one sees that these two schemes are isomorphic.

The proof is then complete once we are able to show the following.

\noindent {\bfseries Claim.} {\it The image of~$B_{\Pi}$ under~$\delta$ is 
contained in the pullback of the two photographic maps $f_{\vec{p}}, f_{\vec{P}} 
\colon C \longrightarrow M_6$.}

\noindent Notice that the fact that $\Pi$ is supposed to be non-equiform 
implies that the M\"{o}bius curves of~$\vec{p}$ and~$\vec{P}$ are different. 
The pullback of the two maps is
\begin{equation*}
  \Bigl\{ (v,w) \in C \times C \, : \, f_{\vec{p}}(v) = f_{\vec{P}}(w) \Bigr\}.
\end{equation*}
The coordinates of~$f_{\vec{p}}(v)$ are obtained by substituting each 
term~$H_{ij}$ by $V_i - W_j$ in the general formula from 
Eq.~\eqref{equation:structure_photographic_map}; for $f_{\vec{P}}(w)$ one 
just needs to consider $W_i - W_j$ instead. Then the pullback is the scheme cut 
out by the $2 \times 2$ minors of the following $2 \times 5$ matrix:
\[
  \begin{pmatrix}
    W_{12, 36, 45} & W_{14, 23, 56} & W_{16, 25, 34} & W_{16, 23, 45} & W_{12, 
34, 56} \\
    V_{12, 36, 45} & V_{14, 23, 56} & V_{16, 25, 34} & V_{16, 23, 45} & V_{12, 
34, 56}
  \end{pmatrix},
\]
where $W_{ij, kl, mn} = (W_i - W_j)(W_k - W_l)(W_m - W_n)$ and similarly for 
$V_{ij, kl, mn}$. A direct computation (for example, with the aid of Gr\"{o}bner 
bases) shows that the ideal generated by such $2 \times 2$ minors is contained 
in the ideal of the $4 \times 4$ minors of~$N_{\Pi}$. This settles the claim and 
hence concludes the proof.
\end{proof} 

\begin{remark}
\label{remark:pullback}
  We cannot hope for equality in the last claim in the proof of 
Theorem~\ref{theorem:bound_degree}. In fact, consider one of the~$15$ 
planes in~$M_6$ (see the Fact in 
Subsection~\ref{degree:photogrammetry:construction}), and suppose it 
parametrizes classes of tuples $(m_1, \dotsc, m_6)$ where $m_i = m_j$; then the 
projection from such a plane maps~$M_6$ to~$\p^1_{\C}$, and the latter has a 
modular interpretation as~$M_4$, namely the moduli space of $4$-tuples $(m_1, 
\dotsc, \widehat{m}_i, \dotsc, \widehat{m}_j, \dotsc, m_6)$ obtained by 
removing~$m_i$ and~$m_j$. Then the images of~$f_{\vec{p}}(v)$ 
and~$f_{\vec{P}}(w)$ under the projection from the plane of classes with $m_s = 
m_t$ coincide if and only if
\begin{multline}
\label{equation:cross_ratio}
  (W_i - W_j)(W_k - W_l)(V_i - V_l)(V_k - V_j) - \\ 
  - (V_i - V_j)(V_k - V_l)(W_i - W_l)(W_k - W_j) \; = \; 0,
\end{multline}
where $\{ i,j,k,l \} \cup \{ s, t \} = \{1, \dotsc, 6\}$. On the other hand, 
the left-hand-side of Eq.~\eqref{equation:cross_ratio} is a $4 \times 4$ 
minor of~$N_{\Pi}$. In fact, if we select the submatrix with rows of index 
$i,j,k$ and~$l$ and we perform column operations we obtain
\[
  \begin{pmatrix}
    0 & 0 & 0 & 1 \\
    W_j V_j - W_i V_i & W_j - W_i & V_j - V_i & 1 \\
    W_k V_k - W_i V_i & W_k - W_i & V_k - V_i & 1 \\
    W_l V_l - W_i V_i & W_l - W_i & V_l - V_i & 1 \\
  \end{pmatrix},
\]
and a direct computation proves the equality. As it is shown in the second 
claim of the proof, the zero locus of these minors is image of~$B_{\Pi}$ 
under~$\delta$. 

On the other hand, there are points in the pullback of~$f_{\vec{p}}$ 
and~$f_{\vec{P}}$ for which not all the cross-ratios of 
Eq.~\eqref{equation:cross_ratio} are equal: suppose in fact the points $p_1, 
p_2$ and~$p_3$ are collinear along the direction~$v$, and the points $P_1, P_2$ 
and~$P_3$ are collinear along the direction~$w$; then $f_{\vec{p}}(v)$ and 
$f_{\vec{P}}(w)$ coincide with one of the nodes of~$M_6$, independently of the 
projections of the other points; hence it is possible to have situations in 
which not all the $15$ cross-ratios coincide, but still we have an intersection 
of the two M\"{o}bius curves.
\end{remark}

Next, we estimate the number of intersections of two M\"{o}bius curves of 
degree~$6$ that satisfy some non-degeneracy conditions, and we show that this 
number is always less than or equal to~$14$ (see 
Proposition~\ref{proposition:bound_intersection_moebius}). To clarify why such a 
bound should hold, let us consider $D_1$ and $D_2$, two M\"{o}bius curves of 
degree~$6$. As we will point out later, there always exist at least two 
quadrics passing through each of them. In general, the quadrics passing 
through~$D_1$ will be different from the ones passing through~$D_2$; up to 
swapping the roles of~$D_1$ and $D_2$, in general there will be a quadric~$Q$ 
containing~$D_1$, but not $D_2$. In this case
\[
  \left| D_1 \cap D_2 \right| \; \leq \; \left| Q \cap D_2 \right| \; \leq \; 
12.
\]
Instead, as we are going to see, we reach $14$ intersections when $D_1$ and 
$D_2$ are the two components of a curve of degree~$12$, complete intersection of 
two quadrics and~$M_6$. Unfortunately, proving that the bound holds in all 
cases requires the analysis of several particular situations\footnote{In 
a first draft of this paper we tried to give an all-embracing argument, but we 
were not able to obtain a complete proof.}, so we break the proof of 
Proposition~\ref{proposition:bound_intersection_moebius} into a sequence 
of lemmata. These intermediate results are essentially of two kinds: either 
they predicate some property of subvarieties of~$\p^4_{\C}$ containing a 
M\"{o}bius curve, or they prove that the bound holds in some particular 
sub-case. 

\smallskip
Recall that we denote by $T_{ij}$ the plane in~$M_6$ parametrizing 
the classes of $6$-tuples $(m_1, \dotsc, m_6)$ in~$\p^1_{\C}$ such that $m_i = 
m_j$.

\begin{lemma}
\label{lemma:moebius_no_plane}
Let $D$ be a M\"obius curve of degree~$6$. Then $D$ cannot be contained in a 
plane.
\end{lemma}
\begin{proof}
Suppose $D \subseteq E$, where $E$ is a plane in~$\p^4_{\C}$. We distinguish 
two cases:
\begin{itemize}
  \item[a.]
    The plane~$E$ is completely contained in~$M_6$. Then $E$ is one of the $15$ 
planes~$T_{ij}$. Write $D = f_{\vec{A}}(C)$, then the projections of the 
points~$A_i$ and~$A_j$ along any direction coincide, and this is possible only 
if $A_i = A_j$, but we only consider $6$-tuples $\vec{A}$ of distinct points, so 
we get an absurd. Notice that in this case we did not make use of the assumption 
on the degree of~$D$.
  \item[b.]
    The plane~$E$ is not contained in~$M_6$. Then $M_6 \cap E$ is at most a 
cubic curve, but this is in contrast with the assumption on the degree of~$D$. 
\qedhere
\end{itemize}
\end{proof}

\noindent Therefore a degree~$6$ M\"{o}bius curve lies on at most one 
hyperplane in~$\p^4_{\C}$.

\begin{lemma}
Let $D_1$ and $D_2$ be two distinct degree~$6$ M\"{o}bius curves, and suppose 
that $D_1$ lies on a hyperplane~$H$. If $D_2$ is not contained in~$H$, then 
$\left| D_1 \cap D_2 \right| \leq 6$.
\end{lemma}
\begin{proof}
  By assumption $\left| D_1 \cap D_2 \right| \leq \left| H \cap D_2 \right| 
\leq 6$.
\end{proof}

From the discussion so far we see that if one of the two M\"{o}bius curves lies 
on a hyperplane, then the bound of~$14$ intersections holds provided that the 
other is not contained in the same hyperplane. Now we consider the latter case, 
and prove that the bound still holds.

\begin{lemma}
\label{lemma:moebius_no_degree2}
Let $D$ be a degree~$6$ M\"{o}bius curve. Then $D$ cannot be contained in a 
surface of degree~$2$ completely contained in~$M_6$.
\end{lemma}
\begin{proof}
The proof is based on a photographic description of surfaces of degree~$2$ 
in~$M_6$ (see also Remark~\ref{remark:pullback}). Let $S \subseteq M_6$ be a 
surface of degree~$2$. Then $S$ spans a 
hyperplane~$H$ in~$\p^4_{\C}$, and by degree reasons $M_6 \cap H = S \cup E$, 
where~$E$ is a plane. Hence $E$ is one of the~$15$ planes~$T_{ij}$ in~$M_6$. 
Consider the rational map $\p^4_{\C} \dashrightarrow \p^1_{\C}$ defined by the 
hyperplanes through~$E$: its restriction to~$M_6$ has a photographic meaning, 
namely it sends the class of a tuple $(m_1, \dotsc, m_6)$ to the class of the 
tuple $(m_1, \dotsc, \hat{m}_{i}, \dotsc, \hat{m}_{j}, \dotsc, m_6)$ --- i.e.\ 
we remove the points~$m_i$ and~$m_j$ --- in~$M_4$, identified with~$\p^1_{\C}$. 
Moreover, by construction $S$ is a fiber of such map. Therefore a M\"{o}bius 
curve $D = f_{\vec{A}}(C)$ of degree~$6$ cannot lie in~$S$, because otherwise 
the cross-ratio of the projections of the points $A_1, \dotsc, \hat{A}_{i}, 
\dotsc, \hat{A}_{j}, \dotsc, A_6$ along any direction in~$\R^3$ would stay 
constant; this is only possible if four points in~$\vec{A}$ are collinear, 
which is excluded by Lemma~\ref{lemma:6_space}. 
\end{proof}

\begin{lemma}
\label{lemma:moebius_hyperplane_irreducile}
  Let $D$ be a degree~$6$ M\"{o}bius curve contained in a hyperplane~$H$. Then 
$H \cap M_6$ is an irreducible cubic surface.
\end{lemma}
\begin{proof}
  Suppose that $Y = H \cap M_6$ is reducible, then either $Y = S \cup E$ where 
$E$ is a plane and $S$ is a surface of degree~$2$, or $Y = E_1 \cup E_2 \cup 
E_3$ where the $E_i$ are planes. In both cases a M\"{o}bius curve, which is 
irreducible, cannot lie in any of the components of~$Y$ because of 
Lemma~\ref{lemma:moebius_no_plane} and Lemma~\ref{lemma:moebius_no_degree2}.
\end{proof}

From now on, for a M\"{o}bius curve~$D$ contained in a hyperplane~$H$ we 
define $L_{ij} = H \cap T_{ij}$. We know that all~$L_{ij}$ are lines (and 
not planes), because by Lemma~\ref{lemma:moebius_hyperplane_irreducile} the 
hyperplane~$H$ cannot contain any of the planes~$T_{ij}$.

\begin{remark}
  Consider a M\"{o}bius curve $D = f_{\vec{A}}(C)$ such that $D \cap T_{ij} 
\cap T_{kl} \neq \emptyset$ for pairwise distinct $i,j,k,l \in \{1, \dotsc, 
6\}$, i.e.\ $\left| \{ i, j, k, l \} \right| = 4$. Then, by construction of the 
photographic map, the lines $\overrightarrow{A_iA_j}$ and 
$\overrightarrow{A_kA_l}$ are parallel.
\end{remark}

\begin{definition}
  We say that a $6$-tuple $\vec{A}$ is \emph{non-parallel} if there do not 
exist pairwise distinct $i,j,k,l \in \{1, \dotsc, 6\}$ such that the lines 
$\overrightarrow{A_iA_j}$ and $\overrightarrow{A_kA_l}$ are parallel. We say 
that a M\"{o}bius curve $D = f_{\vec{A}}(C)$ is non-parallel if $\vec{A}$ is so.
\end{definition}

\begin{lemma}
\label{lemma:lines_distinct}
Let $D$ be a degree~$6$ M\"{o}bius curve contained in a hyperplane~$H$, and 
suppose that $D$ is non-parallel. Then all $15$ lines~$\{ L_{ij} \}$ are 
distinct.
\end{lemma}
\begin{proof}
  Suppose $L_{ij} = L_{kl}$. First of all, notice that this can happen only if 
all $i,j,k,l$ are distinct, since if $\left| i,j,k,l \right| = 3$ with $i \neq 
j$ and $k \neq l$, then $T_{ij}$ and $T_{kl}$ meet only in a point. Then 
\[
  L_{ij} \; = \; L_{kl} \; = \; T_{ij} \cap T_{kl}.
\]
Since $D = f_{\vec{A}}(C)$ meets every $T_{ij}$ in two (not necessarily 
distinct) points --- given by the projections long the directions 
$\overrightarrow{A_iA_j}$ and $\overrightarrow{A_jA_i}$ --- then $D$ would meet 
$L_{ij}$ and $L_{kl}$, but this contradicts the assumption that $D$ is 
non-parallel.
\end{proof}

\begin{lemma}
\label{lemma:intersection_lines}
  Let $D$ be a degree~$6$ M\"{o}bius curve contained in a hyperplane~$H$, and 
suppose that $D$ is non-parallel. Then
\begin{enumerate}[a.]
  \item
  \label{lemma:intersection_lines:4}
    if $\left| \{ i,j,k,l \} \right| = 4$, then $L_{ij} \cap L_{kl} = \{ 
\mathrm{point} \}$;
  \item
  \label{lemma:intersection_lines:3}
    if $\left| \{ i,j,k,l \} \right| = 3$ with $i \neq j$ and $k \neq l$, then 
$L_{ij} \cap L_{kl} = \{ \mathrm{point} \}$ if and only if the hyperplane~$H$ 
passes through the node of~$M_6$ corresponding to configurations $(m_1, \dotsc, 
m_6)$ for which $m_i = m_j = m_k = m_l$; otherwise $L_{ij} \cap L_{kl} = 
\emptyset$.
\end{enumerate}
\end{lemma}
\begin{proof}
  The statement follows from the definition of the lines~$L_{ij}$, from the 
fact that they are all distinct (Lemma~\ref{lemma:lines_distinct}) and from the 
fact that in Case~\ref{lemma:intersection_lines:4} the planes~$T_{ij}$ 
and~$T_{kl}$ intersect in a line, while in 
Case~\ref{lemma:intersection_lines:3} 
the planes~$T_{ij}$ and~$T_{kl}$ intersect in a point, a node of~$M_6$.
\end{proof}

\begin{remark}
  Let $D$ be a degree~$6$ M\"{o}bius curve contained in a hyperplane~$H$, 
and suppose that $D$ is non-parallel. Then $H$ can pass through at most one 
node of $M_6$. In fact, if $H$ passes though two nodes, then $H$ contains a 
line of the form~$T_{ij} \cap T_{kl}$; this implies that $D$ intersects 
$T_{ij} \cap T_{kl}$, which contradicts the hypothesis that $D$ is 
non-parallel.
\end{remark}

We are now going through all the possible cases for an irreducible cubic 
surface $S = H \cap M_6$ obtained by intersecting~$M_6$ with a hyperplane 
containing a degree~$6$ non-parallel M\"{o}bius curve~$D$.

First we analyze the case when $S$ is a singular cubic with non-isolated 
singularities (see for example~\cite[Case~E]{Bruce1979}): here $S$ is either a  
cone or a projection of a cubic scroll in~$\p^4_{\C}$. In the first case, all 
lines of~$S$ meet, but this contradicts Lemma~\ref{lemma:intersection_lines}, 
which asserts that some lines in~$S$ do not meet. In the second case there 
exists a pencil of pairwise disjoint lines, each of which intersects two other 
special lines on the surface; this case is again ruled out by 
Lemma~\ref{lemma:intersection_lines}, which implies the existence of three 
lines intersecting pairwise (for example, $L_{12}$, $L_{34}$ and~$L_{56}$).
Hence none of these cases can appear in our context.

We are left with the case when $S$ is smooth, or has isolated singularities. We 
start with the smooth situation.

\begin{lemma}
\label{lemma:cubic_smooth}
  Let $D_1$, $D_2$ be two distinct degree~$6$ non-parallel M\"{o}bius curves. 
Suppose that both $D_1$ and $D_2$ are contained in a hyperplane~$H$ and that $S 
= H \cap M_6$ is a smooth cubic surface. Then $\left| D_1 \cap D_2 \right| \leq 
14$.
\end{lemma}
\begin{proof}
  Let $D$ be a M\"{o}bius curve as in the hypothesis. Since $S$ is smooth, we 
can express it as the blowup of~$\p^2_{\C}$ at $6$ points $q_1, \dotsc, q_6$ in 
general position such that its Picard group is generated by $\left\langle L, 
E_1, \dotsc, E_6 \right\rangle$ --- where 
each~$E_i$ is the class of the exceptional divisor over~$q_i$, and $L$ is the 
class of the strict transform of a line in~$\p^2_{\C}$, so we have the relations
\[
  L^2 \; = \; 1, \quad E_i^2 \; = \; -1, \quad E_i \cdot L \; = \; 0, \quad E_i 
\cdot E_j \; = \; 0 \; \text{ if } \; i \neq j.
\]
Moreover, denoting by $[ \, \cdot \, ]$ the class in~$\mathrm{Pic}(S)$ of a 
divisor, we can put ourselves in the situation where
\[
  [L_{ij}] \; = \; L - E_i - E_j.
\]
There exist integers $k$ and $e_1, \dotsc, e_6$ such that
\[
  [D] \; = \; k \tth L - (e_1 \tth E_1 + \dotsb + e_6 \tth E_6).
\]
Since $D$ intersects each~$L_{ij}$ in $2$ points,
\[
  k - e_i - e_j \; = \; 2 \qquad \forall i, j.
\]
From this we deduce that
\[
  e_i \; = \; m \quad \forall i, \qquad k = 2m + 2 \qquad \text{for some 
integer } m.
\]
Since $D$ is effective, then $[D] \cdot E_i \geq 0$ for 
all~$i$, and $[D] \cdot (2L - E_1 - \dotsb - E_5) \geq 0$, so $0 \leq m \leq 
4$. We are going to exclude the cases $m = 0$ and $m = 4$. If $m = 0$, then $[D] 
= 2L$. Recall from Lemma~\ref{lemma:moebius_extend} that $D$ is contained in a 
projection of a Veronese surface, which is the complete intersection of~$M_6$ 
with another cubic hypersurface~$U$. Then $U \cap H$ is a cubic surface 
in~$\p^3_{\C}$ containing $D$, therefore $-3K_{S} - [D]$ is effective --- where 
$K_S$ is the canonical divisor on~$S$; recall in fact that $S$ is 
anticanonically embedded in~$\p^3_{\C}$, so $[U] = -3K_{S}$ 
in~$\mathrm{Pic}(S)$. But 
\begin{align*}
  -3 K_S - [D] & = 3 \tth (3L - E_1 - \dotsb - E_6) - 2L \\
  & = 7L - 3(E_1 + \dotsb + E_6).
\end{align*}
On the other hand
\[
  \bigl( 7L - 3(E_1 + \dotsb + E_6) \bigr)(2L - E_1 - \dotsb - E_5) \; < \; 0,
\]
and this is absurd, since both divisors are effective. Similarly for $m = 4$.
Hence $m \in \{ 1, 2, 3 \}$. 

We are ready to prove the statement. By what we said so far, for $i \in \{1, 
2\}$ we have $[D_i] = (2m_i + 2)L - m_i(E_1 + \dotsb + E_6)$ for some $m_i \in 
\{ 1,2,3 \}$. One computes
\[
  [D_1] \cdot [D_2] \; = \; -2(m_1 - 2)(m_2 - 2) + 12.
\]
Then it follows that $[D_1] \cdot [D_2] \leq 14$, so the statement is proved.
\end{proof}

We are left with the situation when the cubic surface $S = H \cap M_6$ has 
isolated singularities. Taking into account that at least $15$ different 
lines lie on~$S$ (see Lemma~\ref{lemma:lines_distinct}), by the classification 
of cubic surfaces (see for example~\cite{Bruce1979}) the only 
possibilities are:
\begin{enumerate}[i.]
  \item 
    a cone over a cubic plane curve (infinitely many lines);
  \item
    one singularity of type~$A_1$ ($21$ lines);
  \item 
    two singularities of type~$A_1$ ($16$ lines);
  \item
    one singularity of type~$A_2$ ($15$ lines).
\end{enumerate}
We consider these cases one by one.

\begin{lemma}
  Let $D_1$, $D_2$ be two distinct degree~$6$ non-parallel M\"{o}bius curves. 
Then it cannot happen that both $D_1$ and $D_2$ are contained in a 
hyperplane~$H$ and that $S = H \cap M_6$ is a cone over a cubic plane curve.
\end{lemma}
\begin{proof}
   This case is ruled out as before by the existence of non-intersecting lines 
on~$S$ (see Lemma~\ref{lemma:intersection_lines}).
\end{proof}

\begin{lemma}
  Let $D_1$, $D_2$ be two distinct degree~$6$ non-parallel M\"{o}bius curves. 
Then it cannot happen that both $D_1$ and $D_2$ are contained in a 
hyperplane~$H$ and that $S = H \cap M_6$ is a singular cubic surface with two
singularities of type~$A_1$.
\end{lemma}
\begin{proof}
  This case is ruled out by the following argument. If $S$ is such a cubic 
surface, then there are exactly $5$ lines passing through each of the two $A_1$ 
singularities. In fact such a surface can be realized as the blowup of 
$\p^2_{\C}$ at $6$ points $q_1, \dotsc, q_6$ --- where both $q_1, q_2, q_3$ 
and $q_1, q_4, q_5$ are collinear, but involving different lines, and $q_6$ 
is in general position with respect to the other points --- followed by blowing 
down the (strict transforms of the) lines $\overrightarrow{p_1p_2p_3}$ and 
$\overrightarrow{p_1p_4p_5}$, which get contracted to the two singularities. 
In this case, using the same notation as in Lemma~\ref{lemma:cubic_smooth}, the 
classes of the (strict transforms of the) $16$ lines of~$S$ in the blowup 
of~$\p^2_{\C}$ at $q_1, \dotsc, q_6$ are
\begin{align*}
  & E_1, \dotsc, E_6, & \text{(6 lines)} \\
  & L - E_1 - E_6, \dotsc, L - E_5 - E_6, & \text{(5 lines)} \\
  & L - E_2 - E_4, L - E_2 - E_5, L - E_3 - E_4, L - E_3 - E_5, &
    \text{(4 lines)} \\
  & 2L - E_2 - E_3 - E_4 -E_5 - E_6. & \text{(1 line)}
\end{align*}
In this way the number of lines that pass through a singular point is given by 
counting how many of the previous ones intersect one of the $(-2)$-curves
\[
  L - E_1 - E_2 - E_3 \quad \text{or} \quad L - E_1 - E_4 - E_5.
\]
However, in our case from Lemma~\ref{lemma:intersection_lines} we deduce that 
either we have $6$ lines passing through a singularity (and this happens when 
the hyperplane~$H$ passes though one of the nodes of~$M_6$) or there are at 
most $2$ lines passing through a point. Hence $S$ cannot have two $A_1$ 
singularities.
\end{proof}

\begin{lemma}
  Let $D_1$, $D_2$ be two distinct degree~$6$ non-parallel M\"{o}bius curves. 
Suppose that both $D_1$ and $D_2$ are contained in a hyperplane~$H$ and that $S 
= H \cap M_6$ is a singular cubic surface with one singularity of type~$A_1$. 
Then $\left| D_1 \cap D_2 \right| \leq 14$.
\end{lemma}
\begin{proof}
   In this case we can think of~$S$ as obtained in the following way: we blow 
up $\p^2_{\C}$ at $6$ points $q_1, \dotsc, q_6$ lying on a conic, and then we 
blow down the $(-2)$-curve $2L - (E_1 + \dotsb + E_6)$, which gets contracted 
to the singular point. The classes of the (strict transforms of the) lines 
on~$S$ are
\begin{align*}
  & E_1, \dotsc, E_6, & \text{(6 lines)} \\
  & L - E_i - E_j. & \text{(15 lines)}
\end{align*}
The latter $15$ lines are the classes of the lines~$L_{ij}$. The computation 
for $[D_1] \cdot [D_2]$ goes exactly as in the smooth case because none of the 
lines~$L_{ij}$ passes through the singular point, as one can check by computing 
the intersection product of $L - E_i - E_j$ with the $(-2)$-curve $2L - (E_1
+ \dotsb + E_6)$. From this we conclude that $[D_k] \cdot [L_{ij}] = 2$ for all 
$i,j$ and for $k \in \{1,2\}$, and so we can proceed as in the smooth case (see 
Lemma~\ref{lemma:cubic_smooth}). However, we cannot directly infer from $[D_1] 
\cdot [D_2]$ the number of intersections of~$D_1$ and~$D_2$. In fact, if $D_1$ 
and $D_2$ intersect in the singular point of~$S$, such intersection is counted 
as a contribution by~$1/2$ --- and not~$1$, as usual --- in $[D_1] \cdot 
[D_2]$. Luckily in this situation we can exclude that $m_1$ or $m_2$ 
equals~$3$. 
In fact, for such a value we obtain the class $8L - 3(E_1 + \dotsb + E_6)$, 
which should hence be effective; on the other hand also the class $2L - (E_1 + 
\dotsb + E_6)$ is effective since the points~$q_i$ are arranged on a conic. But
\[
  \left( 8L - 3 \sum E_i \right)\left( 2L - \sum E_i\right) \; = \; -2 \; < \; 
0,
\]
and this is absurd. Hence $m_i \in \{1,2\}$. We analyze the possible cases.
\begin{itemize}
  \item[a.]
    Either $m_1 = 2$ or $m_2 = 2$, then $[D_1] \cdot [D_2] = 12$. By 
    computing the intersection product with the $(-2)$-curve $2L - (E_1 +
    \dotsb + E_6)$ we see that in this case either~$D_1$ or~$D_2$ does not 
    pass through the singularity, so we can conclude that $\left| D_1 \cap
    D_2 \right| \leq 12$.
  \item[b.]
    Or $m_1 = m_2 = 1$, then $[D_1] \cdot [D_2] = 10$. In this case
    both~$D_1$ and~$D_2$ pass through the singular point, and moreover both
    have a node at that point. From this we conclude that $\left| D_1 \cap D_2
    \right| \leq 14$. \qedhere
\end{itemize}
\end{proof}

\begin{lemma}
  Let $D_1$, $D_2$ be two distinct degree~$6$ non-parallel M\"{o}bius curves. 
Suppose that both $D_1$ and $D_2$ are contained in a hyperplane~$H$ and that $S 
= H \cap M_6$ is a singular cubic surface with one singularity of type~$A_2$. 
Then $\left| D_1 \cap D_2 \right| \leq 12$.
\end{lemma}
\begin{proof}
      In this case we obtain~$S$ by blowing up~$\p^2_{\C}$ at $6$ points $q_1,
    \dotsc, q_6$ such that $q_1, q_2, q_3$ are collinear and $q_4, q_5, q_6$ 
    are collinear on another line, and then blowing down the (strict transforms
    of the) lines~$\overrightarrow{q_1q_2q_3}$ and~$\overrightarrow{q_4q_5q_6}$
    --- the latter get contracted to the unique $A_2$ singularity of~$S$. In
    this case the classes of the (strict transforms of the) $15$ lines of~$S$, 
    which coincide with the lines~$L_{ij}$, are
    \begin{align*}
      & E_1, \dotsc, E_6, & \text{(6 lines)} \\
      & L - E_1 - E_4, \dotsc, L - E_1 - E_6, & \text{(3 lines)} \\
      & L - E_2 - E_4, \dotsc, L - E_2 - E_6, & \text{(3 lines)} \\ 
      & L - E_3 - E_4, \dotsc, L - E_3 - E_6. & \text{(3 lines)} 
    \end{align*}
    Notice that the only lines~$L_{ij}$ passing through the singular point are
    the ones whose class is an exceptional divisor~$E_i$, as the computation
    of the intersection product with the two $(-2)$-classes $L - E_1 - E_2 - 
E_3$
    and $L - E_4 - E_5 - E_5$ confirms.
    Here the computation of~$[D_1] \cdot [D_2]$ does not work as in the
    smooth case, due to the fact that in the blowup of~$\p^2_{\C}$ at the points
    $\{ q_i \}$ the (strict transforms of) the curves~$D_i$ may not intersect 
the
    (strict transforms of) the lines~$L_{ij}$ --- this depends whether or not 
    the curves~$D_i$ pass through the singular points. 

    Let $D$ be a M\"{o}bius curve with the properties of~$D_1$ and~$D_2$. 
    Let us suppose that $D$ does not pass through the singular point of~$S$.
    Then we know that in the blowup of~$\p^2_{\C}$ at the points~$\{ q_i \}$
    we have $[D] \cdot [L_{ij}] = 2$ for all $i \neq j$. If we write
    \[
      [D] \; = \; k \tth L - (e_1 \tth E_1 + \dotsb + e_6 \tth E_6),
    \]
    then the previous conditions translate into
    \[
      \begin{array}{ll}
        [D] \cdot E_i = 2 & \text{for all } i \in \{1, \dotsc, 6\}, \\ {}
        [D] \cdot (L - E_i - E_j) = 2 & \text{for all } i \in \{1,2,3\}, 
        \, j \in \{4,5,6\}.
      \end{array}
    \]
    This forces $k = 6$ and $e_i = 2$ for all~$i$, so $[D] = 6L - 2(E_1 + \dotsb 
+ E_6)$. 
    Suppose now that $D$ passes through the singular point. Notice that the 
    fact that both~$D$ and each~$L_{ij}$ are real implies that their 
    intersection is real;
    moreover, the fact that $D = f_{\vec{A}}(C)$ where $f_{\vec{A}}$ is a real 
map
    and $C$ is a real variety without real points implies that it cannot happen 
    that $D$ intersects an~$L_{ij}$ transversely at the singular point and 
    then in another different point, or tangentially at the singular point. 
    The only possibility is that $D$ has an ordinary node at the singular 
    point. 
    As we pointed out before, the~$L_{ij}$ passing through the singular point 
    are the ones whose class is some~$E_i$ --- this is compatible with the 
    situation when the hyperplane~$H$ passes through a node of~$M_6$, since 
    as pointed out in Case~iii. here we have $6$ lines passing through the 
    node. Therefore
    \[
      [D] \cdot E_i \; = \; 0 \quad \text{for all } i \in \{1, \dotsc, 6\}.
    \]
    Moreover in this case the strict transform of~$D$  meets the two 
$(-2)$-curves
    in two points, so
    \[
      [D] \cdot \Bigl( (L - E_1 - E_2 - E_3) + (L - E_4 - E_5 - E_6) \Bigr) \; = 
\; 2.
    \]
    The result is that $[D] = L$. Summing up, we have the following scenarios:
    \begin{itemize}
      \item[a.]
        Both $D_1$ and $D_2$ do not pass through the singular point. Then 
$[D_1] 
        \cdot [D_2] = 12$, so $\left| D_1 \cap D_2 \right| \leq 12$.
      \item[b.]
        Only $D_1$ (or $D_2$) passes through the singular point. Then $[D_1] 
        \cdot [D_2] = 6$ and so $\left| D_1 \cap D_2 \right| \leq 6$.
      \item[c.]
        Both $D_1$ and $D_2$ pass through the singular point, and have a node 
        there. Then $[D_1] \cdot [D_2] = 1$, and $\left| D_1 \cap D_2 \right| 
        \leq 1+4 = 5$. \qedhere
    \end{itemize}
\end{proof}

The discussion so far proves that the bound on the intersection of two non-parallel 
degree~$6$ M\"{o}bius curves holds if one of them is contained in a hyperplane 
of~$\p^4_{\C}$. Hence from now on we can suppose that both curves~$D_1$ 
and~$D_2$ do not lie on any hyperplane. Riemann-Roch predicts that there are at 
least $2$ quadrics in the ideal of a smooth rational curve of degree~$6$ 
in~$\p^4_{\C}$ (for more details, see the proof of 
Lemma~\ref{lemma:photo_gen_prop}); this is true also in the singular case, 
because in that situation we have an injective homomorphism 
$\mathrm{H}^0\bigl(D_i, \mscr{O}_{D_i}(2)\bigl) \longrightarrow  
\mathrm{H}^0\bigl(\p^1_{\C}, \mscr{O}_{\p^1_{\C}}(12)\bigl)$. Thus each~$D_i$ 
is contained in a pencil of quadrics.

\begin{lemma}
\label{lemma:base_locus_quadrics}
  Let $D$ be a degree~$6$ non-parallel M\"{o}bius curve not contained in any
hyperplane. Let $\mcal{Q}$ be a pencil of quadrics containing~$D$. Then the 
base locus of~$\mcal{Q}$ is a quartic surface.
\end{lemma}
\begin{proof}
  If the base locus of~$\mcal{Q}$ were of dimension~$3$, then it would be a 
component of each of the quadrics in~$\mcal{Q}$, hence all of them would split 
into two hyperplanes. But by assumption $D$ is not contained in any hyperplane, 
so the base locus has dimension~$2$. Then the latter is a complete 
intersection, and from Bezout's theorem it has degree~$4$. 
\end{proof}

In Lemma~\ref{lemma:bound_complete_intersection} and 
Lemma~\ref{lemma:quartic_reducible} we discuss the situation when the base 
locus~$S$ of a pencil of quadrics passing through a M\"{o}bius curve is 
irreducible; in particular we prove that the bound holds when $S$ is not 
contained in~$M_6$, while the case $S \subseteq M_6$ cannot occur. 

\begin{lemma}
\label{lemma:bound_complete_intersection}
  Let $D_1$ and $D_2$ be two distinct degree~$6$ non-parallel M\"{o}bius curves
not contained in any hyperplane. Suppose that there is a pencil~$\mcal{Q}$ of 
quadrics containing~$D_1$ whose base locus~$S$ is irreducible and not contained 
in~$M_6$. Then $\left| D_1 \cap D_2 \right| \leq 14$. 
\end{lemma}
\begin{proof}
   Here $D_1$ is a component of the curve $Z = S \cap M_6$, which is a 
complete intersection of degree~$12$ by Lemma~\ref{lemma:base_locus_quadrics} 
and Bezout's theorem. If $D_2$ coincides with the other 
component of~$Z$, then the second part of Lemma~\ref{lemma:curves_intersection} 
proves that $\left| D_1 \cap D_2 \right| = 14$. If this is not the case, then 
there exists at least a quadric~$Q$ passing through~$D_1$, but not passing 
through~$D_2$. Hence
\[
  \left| D_1 \cap D_2 \right| \; \leq \; \left| Q \cap D_2 \right| \; \leq \; 
12. \qedhere
\]
\end{proof}

\begin{lemma}
\label{lemma:quartic_reducible}
  Let $D$ be a degree~$6$ non-parallel M\"{o}bius curve not contained in any
hyperplane. Let $\mcal{Q}$ be a pencil of quadrics containing~$D$ and suppose 
that its base locus~$S$ is contained in~$M_6$. Then $S$ is reducible.
\end{lemma}
\begin{proof}
  Suppose instead that $S$ is irreducible. Pick any quadric~$Q$ in the 
pencil~$\mcal{Q}$: then the intersection $Q \cap M_6$ is the union of~$S$ and a 
surface~$S'$ of degree~$2$. We claim that it is always possible to choose~$Q$ 
such that $S'$ splits in the union of two planes. In fact, each~$S'$ spans a 
hyperplane~$H$, so that $H \cap M_6 = S' \cup E$, where $E$ is a plane. Since 
the set of planes in~$M_6$ is discrete, by continuity we obtain that the 
plane~$E$ is always the same, regardless of which $Q \in \mcal{Q}$ we start 
with. Therefore the one-dimensional family of surfaces~$S'$ is obtained by 
cutting~$M_6$ with the pencil of hyperplanes through the plane~$E$. A direct 
computation --- for example taking the plane $x_0 = x_4 = 0$, where the~$x_i$ 
are coordinates in~$\p^4_{\C}$ --- shows that in such one-dimensional family 
there are always reducible members. Hence we can select $Q \in \mcal{Q}$ such 
that, after a possible rearrangement of the indices,
\[
  Q \cap M_6 \; = \; S \cup T_{12} \cup T_{34}.
\]
We intersect both sides of the previous equality with the plane~$T_{56}$: 
on the left we obtain either a conic (if $T_{56}$ is not contained in~$Q$) or 
the plane~$T_{56}$ itself (if $T_{56} \subseteq Q$), while on the right we get 
the union of $S \cap T_{56}$, $T_{12} \cap T_{56}$ (a line) and $T_{34} \cap 
T_{56}$ (another line). Since $S$ is irreducible by assumption, then $T_{56}$ 
cannot be contained in~$S$, so $S \cap T_{56}$ can be either a curve, or a 
finite set of points. This forces the left hand side $Q \cap T_{56}$ to be a 
conic. In turn, this implies that $S \cap T_{56}$ has to be contained in the 
union of the two lines~$T_{12} \cap T_{56}$ and~$T_{34} \cap T_{56}$. On the 
other hand, since $D \subseteq S$ and $D$ intersects $T_{56}$ in two points, 
then $D$ should intersect one of the two lines $T_{12} \cap T_{56}$ and $T_{34} 
\cap T_{56}$, but this contradicts the assumption that $D$ is non-parallel. 
Hence $S$ cannot be irreducible.
\end{proof}

Hence we are left with the case when the base locus~$S$ is reducible. We notice 
that it cannot happen that $S$ is contained in~$M_6$ and at the same time 
splits into the union of two surfaces of degree~$2$, because by 
Lemma~\ref{lemma:moebius_no_degree2} no M\"{o}bius curve of degree~$6$ can lie 
on a degree~$2$ surface contained in~$M_6$. The only remaining cases are when 
$S = S' \cup S''$ with both $S'$ and $S''$ surfaces of degree~$2$, but $S 
\not\subseteq M_6$, or $S = E \cup S'$ where $E$ is a plane and $S'$ is a 
cubic surface.

\begin{lemma}
  Let $D_1$ and $D_2$ be two distinct degree~$6$ non-parallel M\"{o}bius curves
not contained in any hyperplane. Suppose that there is a pencil~$\mcal{Q}$ of 
quadrics containing~$D_1$ whose base locus~$S$ is not contained in~$M_6$ and 
splits into the union $S = S' \cup S''$ of two surfaces of degree~$2$. Then 
$\left| D_1 \cap D_2 \right| \leq 14$.
\end{lemma}
\begin{proof}
  Since $D_1$ is irreducible, then $D_1 \subseteq S'$ or $D_1 \subseteq S''$. 
From now on we will suppose $D_1 \subseteq S'$. Since $D_1$ cannot lie in a 
degree~$2$ surface contained in~$M_6$, then $S'$ is not contained in~$M_6$ and 
hence by degree reasons $D_1 = S' \cap M_6$. Suppose that there exists a 
quadric $Q \in \mcal{Q}$ not passing through~$D_2$; then
\[
  \left| D_1 \cap D_2 \right| \; \leq \; \left| Q \cap D_2 \right| \; \leq \; 
12. 
\]
Otherwise $D_2 \subseteq S$. It cannot happen that $D_2 \subseteq S'$, because 
otherwise we would have $D_1 = D_2$, contradicting the hypothesis. Thus $D_2 
\subseteq S''$, and so $D_2 = S'' \cap M_6$. Therefore $D_1$ and $D_2$ are the 
two components of the degree~$12$ complete intersection $Z = S \cap M_6$, and 
then the second part of Lemma~\ref{lemma:curves_intersection} concludes the 
proof.
\end{proof}

\begin{lemma}
   Let $D_1$ and $D_2$ be two distinct degree~$6$ non-parallel M\"{o}bius 
curves not contained in any hyperplane. Suppose that there is a 
pencil~$\mcal{Q}$ of quadrics containing~$D$, whose base locus~$S$ 
splits into the union of a plane~$E$ and a cubic surface~$S'$. Then $\left| D_1 
\cap D_2 \right| \leq 14$.
\end{lemma}
\begin{proof}
  By Lemma~\ref{lemma:moebius_no_plane}, the curve~$D_1$ cannot be contained in 
the plane~$E$, so $D \subseteq S'$. Suppose that $S'$ is not contained 
in~$M_6$. There are several possibilities.
\begin{itemize}
  \item[a.]
    The intersection $Z = S' \cap M_6$ is a curve of degree~$9$. Then $D_1$ is 
a component of such curve. This implies that there is a quadric $Q \in 
\mcal{Q}$ not passing through $D_2$, because otherwise we would have $D_1 = 
D_2$. Hence 
\[
  \left| D_1 \cap D_2 \right| \; \leq \; \left| Q \cap D_2 \right| \; \leq \; 
12. 
\]
  \item[b.]
    The cubic $S'$ splits into a plane~$E'$ and a surface~$S''$ of degree~$2$, 
and $S''$ is contained in~$M_6$. This case cannot happen, since $D_1$ neither 
can lie on a plane, nor on a degree~$2$ surface contained in~$M_6$.
  \item[c.]
    The cubic $S'$ splits into a plane~$E'$ and a surface~$S''$ of degree~$2$, 
and $S''$ is not contained in~$M_6$. Then $D_1 = S'' \cap M_6$. This implies 
that there exists a quadric $Q \in \mcal{Q}$ not passing through $D_2$, because 
otherwise we would have $D_1 = D_2$. Hence 
\[
  \left| D_1 \cap D_2 \right| \; \leq \; \left| Q \cap D_2 \right| \; \leq \; 
12. 
\]
\end{itemize}
The last case that needs to be treated is the one where $S'$ is contained 
in~$M_6$. Then $S'$ is irreducible, because $D'$ cannot lie on planes or 
surfaces of degree~$2$ contained in~$M_6$ (see 
Lemma~\ref{lemma:moebius_no_plane} and~\ref{lemma:moebius_no_degree2}). Hence 
$S'$ is a cubic scroll --- maybe singular, namely a cone over a rational cubic 
plane curve. Thus $S'$ admits a determinantal representation as the zero set of 
the $2 \times 2$ minors of a $2 \times 3$ matrix of linear forms. We consider 
the intersection of $S'$ with the planes~$T_{ij}$: first of all, each~$T_{ij}$ 
is not contained in~$S'$, because otherwise $S'$ would be reducible. By 
restricting the determinantal representation of~$S'$ to $T_{ij}$ we see that 
$L_{ij} = S'\cap T_{ij}$ is defined by three quadratic equations in~$T_{ij}$, so 
it is either a finite set of points, or a line, or a conic. On the other hand 
if $\{i,j,k,l,m,n\} = \{1, \dotsc, 6\}$ then there exists a hyperplane~$H$ 
in~$\p^4_{\C}$ such that $H \cap M_6 = T_{ij} \cup T_{kl} \cup T_{mn}$. 
Therefore 
\[ 
  H \cap S' \; = \; (S' \cap T_{ij}) \cup (S' \cap T_{kl}) \cup (S' \cap 
T_{mn}) \; = \; L_{ij} \cup L_{kl} \cup L_{mn}.
\]
Since $S'$ is not contained in~$H$ (because $D_1$ does not lie in any 
hyperplane) we have that $S' \cap H$ is a cubic curve. Suppose that one among 
$L_{ij}$, $L_{kl}$ and $L_{mn}$, say $L_{ij}$, is a finite number of points. 
Then, by eventually rearranging the indices, $L_{kl}$ is a line and $L_{mn}$ is 
a conic. This implies that $L_{ij} \subseteq L_{kl} \cup L_{mn}$, but this 
contradicts the hypothesis on $D_1$ of being non-parallel. Therefore we 
conclude that all~$L_{ij}$ are lines. Moreover all lines~$L_{ij}$ are distinct, 
because otherwise this would violate the non-parallel assumption. We consider 
now the mutual position of the lines~$L_{ij}$. First of all we notice that ---
analogously as in Lemma~\ref{lemma:intersection_lines} --- if $\left| 
\{i,j,k,l\} \right| = 3$ and $L_{ij}$ meets $L_{kl}$, then they 
intersect at the node of~$M_6$ given by the class of $6$-tuples $(m_1, \dotsc, 
m_6)$ for which $m_i = m_j = m_k = m_l$. This rules out the case where $S'$ is 
a cone, because in that case all lines in~$S'$ meet in a single point, but on 
the other hand the points $L_{12} \cap L_{23}$ and $L_{12} \cap L_{24}$ are 
different, an absurd. So $S'$ is smooth, then $S'$ admits a pencil of mutually 
disjoint lines, all of which intersect one line~$\ell$. We distinguish two 
cases:
\begin{itemize}
  \item[a.]
    Suppose that $\ell$ does not appear among the lines~$L_{ij}$. Then in 
particular $L_{12}$, $L_{34}$ and $L_{56}$ are disjoint. On the other hand, as 
mentioned before, $L_{12} \cup L_{34} \cup L_{56} = H \cap M_6$ for some 
hyperplane~$H$; from~\cite[Chapter~III, Corollary~7.9]{Hartshorne1977} we 
know that a hyperplane section of~$S'$ is connected, so this is absurd.
  \item[b.]
    Suppose that $\ell$ appears among the lines~$L_{ij}$. After a possible 
rearrangement of the indices, we can suppose that $\ell = L_{12}$. Then 
$L_{23}$, $L_{45}$ and $L_{16}$ are disjoint, but we can repeat the argument of 
Case a. and see that this leads to an absurd.
\end{itemize}
This concludes the proof of the statement.
\end{proof}

\noindent We sum up the previous discussion in the following proposition.

\begin{proposition}
\label{proposition:bound_intersection_moebius}
  Let $D_1$ and $D_2$ be two distinct degree~$6$ non-parallel M\"{o}bius 
curves. Then $\left| D_1 \cap D_2 \right| \leq 14$.
\end{proposition}

\noindent Now we can prove the main result of this section.

\begin{theorem}
\label{thm:bound_degree}
The conformal degree of a non-planar and non-equiform hexapod~$\Pi$ such that 
both base and platform are non-parallel is at most~$28$. 
\end{theorem}
\begin{proof}
	Since $\Pi$ is non-planar, then the M\"{o}bius curves $D_1$ and $D_2$ of base 
and platform points are birational by Lemma~\ref{lemma:6_space}; by the same 
result, using the non-parallel hypothesis we infer that that no $4$ points are 
on a line, and so we conclude that both $D_1$ and $D_2$ have degree~$6$. Since 
$\Pi$ is non-equiform, then $D_1$  and $D_2$ are distinct. 
Theorem~\ref{theorem:bound_degree} states that the conformal degree 
of~$\Pi$ is bounded by $2 \tth \left| D_1 \cap D_2 \right|$, and 
Proposition~\ref{proposition:bound_intersection_moebius} asserts 
that $\left| D_1 \cap D_2 \right| \leq 14$, so the statement is proved.
\end{proof}

\section{Construction of liaison hexapods}
\label{liaison_hexapods}

The goal of this section is to provide a construction for a family of movable 
hexapods, that we will call \emph{liaison hexapods}. This will be accomplished 
in two stages: first for each general choice a base we design a 
candidate for the platform (Subsection~\ref{liaison_hexapods:platform}), then 
we compute a dilation of the latter and leg lengths that guarantee mobility for 
the corresponding hexapod (Subsection~\ref{liaison_hexapods:leg_lengths}). We 
believe that the family created in this way is maximal among movable hexapods,
namely, if we consider the set of movable hexapods as an algebraic variety,
it is not contained in any irreducible component of strictly larger dimension.

\subsection{The candidate platform}
\label{liaison_hexapods:platform}

We begin with the definition of the candidate platform, once we are given a 
base constituted of $6$ general points.

The idea for the construction is the following: in 
Subsection~\ref{degree:photogrammetry} we associated to any $6$-tuple  
of distinct points in~$\R^3$ its M\"{o}bius curve in the moduli 
space~$M_6$. At this point, we know from bond theory and from the discussion in 
Section~\ref{degree} that the more the M\"{o}bius curves associated to base and 
platform of a hexapod intersect, the more the hexapod has the chance to be 
movable and the higher will be its conformal degree. Inspired by this, we start 
from the curve~$D$ associated to a given general $6$-tuple~$\vec{P}$ 
of points in~$\R^3$, and we 
construct from it, applying \emph{liaison} techniques, another curve~$D'$, for 
which we give evidence to be the curve associated to another tuple~$\vec{p}$ 
of~$6$ points in~$\R^3$ (unfortunately we are not able to exhibit a complete 
proof of this), and that intersects~$D$ in $14$ points. 

At this stage, the tuple~$\vec{p}$ is only determined up to similarities, namely
rotations, dilations and translations, and the right scaling factor will be
fixed later in Subsection~\ref{liaison_hexapods:leg_lengths}.

\smallskip
We briefly glance at the concept of liaison via an example. Let $D \subseteq 
\p^3_{\C}$ be a projective curve. We know that $I(D)$, the homogeneous ideal 
of~$D$, cannot 
be generated by less than~$2$ polynomials. Hence we can always pick two
homogeneous polynomials $f, g \in I(D)$ such that the corresponding surfaces
$F = V(f)$ and $G = V(g)$ intersect in a one-dimensional
projective set. Since we took $f$ and $g$ in the ideal of~$D$, we have that $D
\subseteq F \cap G$. The inclusion may be strict (unless $D$ is a
so-called \emph{complete intersection}), and in that case the intersection 
$F \cap G$ is the union of $D$ and another curve~$D'$. We say that the 
curves~$D$ and~$D'$ are \emph{linked} by $Y = F \cap G$. This 
procedure can be applied not only to curves in~$\p^3_{\C}$, but also to 
curves (and higher dimensional varieties) in any projective space, as we are 
going to do for the case of~$\p^4_{\C}$. Linked curves share many properties, 
and in particular we are interested in a relation between their degrees and 
genera, expressed by the following result (see~\cite[Chapter~3, 
Corollary~5.2.14]{Migliore}):

\begin{proposition}
\label{proposition:linked_genera}
  Let $D$ and $D'$ be two projective curves in~$\p^4_{\C}$ linked by a complete 
intersection $Y = F_1 \cap F_2 \cap F_3$ and let $p_a(D)$ and 
$p_a(D')$ be their arithmetic genera. Then
\[ 
  p_a(D) - p_a(D') \; = \; \frac{1}{2}(d - 5)(\deg{D} - \deg{D'}),
\]
	where $d = \deg{F_1} + \deg{F_2} + \deg{F_3}$. 
\end{proposition}
Proposition~\ref{proposition:linked_genera} implies in particular that if the 
curves~$D$ and~$D'$ have the same degree, then they have the same genus. 
This property will be used in the proof of 
Lemma~\ref{lemma:curves_intersection}.

As it will be clear from several proofs in this section, the construction we 
are going to propose works only if the base points to start with are 
sufficiently general. We would like to make this condition precise, in 
order to provide later a conjecture that is easily falsifiable. 

\begin{definition}
We say that a $6$-tuple of points in~$\R^3$ is \emph{M\"{o}bius-general} if 
its M\"{o}bius curve is smooth, the ideal of the M\"{o}bius curve contains only 
two linearly independent quadratic forms, and these two quadratic forms cut out 
a one-dimensional set from~$M_6$ consisting of two smooth curves. 
\end{definition}

The following lemma, together with Remark~\ref{remark:residual_smooth}, shows 
that M\"{o}bius-general $6$-tuples form an open subset of the variety of $6$-tuple 
of points in~$\R^3$.

\begin{lemma}
\label{lemma:photo_gen_prop}
	Let $\vec{A}$ be a general $6$-tuple of points in~$\R^3$. Then the 
M\"{o}bius curve of~$\vec{A}$ is a smooth curve of degree~$6$ contained in
the complete intersection of~$M_6$ and two quadric hypersurfaces. 
\end{lemma}
\begin{proof}
	Let $D$ be the M\"{o}bius curve of~$\vec{A}$. Since $\vec{A}$ is
general, we can suppose that it is non-planar, so from Lemma~\ref{lemma:6_space}
the degree of~$D$ is~$6$ (because we have degree~$4$ only for
a special choice of~$\vec{A}$). We prove the smoothness of~$D$ with the 
following argument. What we showed so far is that for a general~$\vec{A}$ the 
curve~$D$ is rational and of degree~$6$; therefore, it can be thought as a 
point $[D]$ in the Hilbert scheme $\operatorname{Hilb}(\p^4_{\C}, 6t+1)$ of 
subschemes of~$\p^4_{\C}$ with Hilbert polynomial~$6t+1$. Hence we have a 
map $\xi \colon \mcal{V} \longrightarrow \operatorname{Hilb}(\p^4_{\C}, 6t+1)$, 
where $\mcal{V}$ is a suitable Zariski-open subset of~$\left(\R^3\right)^6$: 
the map~$\xi$ associates to a $6$-tuple $\vec{A}$ the image~$[D]$ of its 
photographic map~$f_{\vec{A}}$. Since smooth rational sextics form an open 
subset of~$\operatorname{Hilb}(\p^4_{\C}, 6t+1)$, if we are able to show that 
for a particular $6$-tuple $\vec{A}$ the curve~$D$ is smooth, then for all 
$\vec{A}$ belonging to a (possibly smaller) open set $\mcal{W} \subseteq 
\mcal{V}$ the M\"{o}bius curve of~$\vec{A}$ is smooth. One can check that if we 
take~$\vec{A}$ with
\begin{align}\label{A123}
  A_1& = (0,0,0), &\,\,  A_2&=(2,0,0), &\,\,  A_3&=(3,2,0), \\ \label{A456}
  A_4& = (2,3,1), &\,\,  A_5&=(1,2,2), &\,\,  A_6& = (3,1,3)
\end{align}
then the obtained curve~$D$ is a smooth sextic.

Since $D$ is smooth and rational, Riemann-Roch implies that 
$\operatorname{h}^0(D, \mscr{O}_{D}(2)) = 13$, i.e.\ there is at most a  
$13$-dimensional family of quadrics cutting~$D$ in finitely many point. There is 
a $15$-dimensional family of quadrics in $\p^4_{\C}$, so there are at least two 
linearly independent quadrics passing through~$D$. Since in the example we just 
gave we have exactly two quadrics, then this holds for a 
general~$\vec{A}$, and the same is true for the fact that the two quadrics 
form a complete intersection with~$M_6$.
\end{proof}

\begin{definition}
	Let $\vec{A}$ be a M\"{o}bius-general $6$-tuple of points in~$\R^3$. Let $D$ 
be the M\"{o}bius curve of~$\vec{A}$ and let $Y$ be the complete intersection 
of degree~$12$ whose existence is ensured by Lemma~\ref{lemma:photo_gen_prop}. 
The curve~$D'$ such that $D \cup D' = Y$ is called the \emph{residual curve} 
of~$D$. 
\end{definition}

\begin{remark}
\label{remark:residual_smooth}
   The condition that also the residual curve~$D'$ is smooth is an open 
condition, and this can be proved as in Lemma~\ref{lemma:photo_gen_prop}, 
namely showing that this is true in one example, because smoothness is an open 
property; one can check that the same $6$-tuple $\vec{A}$ that we chose in 
Lemma~\ref{lemma:photo_gen_prop} yields a curve whose residual one is smooth. 
This concludes the proof, initiated in Lemma~\ref{lemma:photo_gen_prop}, that a 
general $6$-tuple $\vec{A}$ is M\"{o}bius-general.
\end{remark}

\begin{lemma}
\label{lemma:curves_intersection}
	Let $\vec{A}$ be a M\"{o}bius-general $6$-tuple of points in~$\R^3$. Let $D$ 
be the M\"{o}bius curve of~$\vec{A}$ and let~$D'$ be its residual 
curve. Then $D'$ is rational and of degree~$6$. Moreover $D$ and 
$D'$ intersect in~$14$ points.
\end{lemma}
\begin{proof}
Since $\vec{A}$ is M\"{o}bius-general, then the curve~$D'$ has degree~$6$. From 
Proposition~\ref{proposition:linked_genera} we obtain $p_{a}(D) = p_{a}(D') = 
0$. From \cite[Remark~5.2.7]{Migliore} we have the following exact sequence:
\[
  \xymatrix{0 \ar[r] & \omega_{D}(-2) \ar[r] & \mscr{O}_{D \cup D'} \ar[r] & 
\mscr{O}_{D'} \ar[r] & 0}
\]
where~$\omega_{D}$ denotes the canonical sheaf of~$D$. Taking the associated 
long exact sequence in cohomology and using the fact that $p_{a}(D') = 0$, one 
can prove that $\operatorname{h}^0(D', \mscr{O}_{D'}) = 1$, namely $D'$ is 
connected. Since $D'$ is smooth by the assumption of M\"{o}bius-generality, it 
is irreducible, and so rational.

We are left to prove that $D$ and $D'$ intersect in~$14$ points. Since $D \cup 
D'$ is a complete intersection whose ideal is generated by two quadrics and a 
cubic, the ideal sheaf $\mscr{I}_{D \cup D'}$ admits a graded free resolution 
given by the Koszul complex:
\[
  \xymatrix@C=.42cm{0 \ar[r] & \mscr{O}_{\p^4}(-7) \ar[r] & 
\mscr{O}_{\p^4_{\C}}(-4) 
\oplus \mscr{O}_{\p^4_{\C}}(-5)^2 \ar[r] & \mscr{O}_{\p^4_{\C}}(-2)^2 \oplus 
\mscr{O}_{\p^4_{\C}}(-3) \ar[r] & \mscr{I}_{D \cup D'} \ar[r] & 0}
\]
A computation shows that the Euler characteristic $\chi \left(\p^4_{\C}, 
\mscr{I}_{D \cup D'} \right)$ equals $13$, which implies $\chi \left(\p^4_{\C}, 
\mscr{O}_{D \cap D'} \right) = 14$ because of the exact sequence
\[
  \xymatrix{0 \ar[r] & \mscr{O}_{D \cap D'} \ar[r] & \mscr{O}_{D} \oplus 
\mscr{O}_{D'} \ar[r] & \mscr{O}_{D \cup D'} \ar[r] & 0}
\]
This proves that $D \cap D'$ is constituted of~$14$ points.
\end{proof}

\begin{remark}
\label{remark:reality_linked}
  Suppose $D$ is the M\"{o}bius curve of a $6$-tuple $\vec{A}$, and $D'$ is its 
residual curve. Then $D'$ inherits a real structure from~$D$. In fact, 
by construction $D$ is a real curve, so the generators of its ideal can be 
taken to be all real polynomials; hence the degree~$12$ complete 
intersection~$Y$ is also a real curve, thus by construction $D'$ is a real 
curve.
\end{remark}

Lemma~\ref{lemma:photo_gen_prop}, together with 
Lemma~\ref{lemma:curves_intersection} and Remark~\ref{remark:reality_linked}, 
shows that if we start from a M\"{o}bius-general $6$-tuple, then the residual 
curve~$D'$ that we obtain satisfies many of the properties that a M\"{o}bius 
curve needs to have (see also 
Proposition~\ref{proposition:characterization_moebius}). Unfortunately, we are 
not able to establish theoretically that $D'$ is actually a M\"{o}bius curve. 
On the other hand, we have strong experimental evidence that this holds; this 
leads us to formulate the following conjecture.

\begin{conjlemma}
\label{clemma:points}
	Let $\vec{A}$ be a M\"{o}bius-general $6$-tuple of points in~$\R^3$. Let 
$D$ be the M\"{o}bius curve of~$\vec{A}$ and let~$D'$ be its residual 
curve. Then $D'$ is a real variety without real points and is contained in a 
linear projection, defined by real polynomials, of the third Veronese embedding 
of~$\p^2_{\C}$. Moreover there exists a $6$-tuple~$\vec{B}$ in~$\R^3$ 
(unique up to translations, rotations and dilations) such that the M\"{o}bius 
curves of~$\vec{A}$ and~$\vec{B}$ intersect in~$14$ points.
\end{conjlemma}

\begin{example} 
\label{example:georg}
If we take $\vec{A}$ as in the proof of Lemma~\ref{lemma:photo_gen_prop}, then
the residual curve~$D'$ is a M\"obius curve. In fact, $D'$ is rational 
and real without real points, so we can compute a real isomorphism $f \colon C 
\longrightarrow D'$, where $C = \{ (x:y:z) \, : \, x^2 + y^2 + 
z^2 = 0 \}$. One notices then that $f$ extends to a real morphism $F \colon 
\p^2_{\C} \longrightarrow M_6$. The preimages of the planes $T_{ij}$ are lines 
in~$\p^2_{\C}$, and their normal vectors give the difference vectors $B_i-B_j$, 
for a $6$-tuple $\vec{B}$ of candidate points, where $i,j \in \{ 1, \dotsc, 6 
\}$ with $i \neq j$. In this way it is not difficult to get possible coordinates 
for the~$B_i$, for instance:
\begin{align} \label{B12}
  B_1 &= (0,0,0), &\,\, B_2&=\left(\tfrac{1397624}{806205}, 
-\tfrac{92216}{161241}, -\tfrac{437432}{806205}\right), \\ \label{B34}
  B_3 &= \left(\tfrac{340244}{161241}, \tfrac{82388}{53747}, 
-\tfrac{835486}{483723} \right), &\,\,
  B_4 &= \left(\tfrac{1341708}{1236181}, \tfrac{3724594}{1236181}, 
-\tfrac{922514}{1236181} \right), \\  \label{B56}
  B_5 &= \left(\tfrac{1125372}{2203627}, \tfrac{5582884}{2203627}, 
\tfrac{2416984}{2203627} \right), &\,\,
  B_6 &= \left(\tfrac{1719522}{591217}, \tfrac{824050}{591217}, 
\tfrac{1683982}{591217} \right).
\end{align}
Eventually one checks that $f = f_{\vec{B}}$, namely $f$ is a photographic map, 
so $D'$ is indeed a M\"{o}bius curve.
\end{example}

\begin{remark}
\label{remark:faith}
If the above conjectured lemma is true for a generic choice of base points, then
it is true for all M\"{o}bius-general $6$-tuples of base points, because within 
the M\"{o}bius-general instances, the statement that the linked curve is 
contained in a linear projection of a Veronese is a closed property.

Despite we do not know whether the conjecture is true or false, we have tested it
using random numerical instances of base points, and all the results confirmed 
it. It could be that we were always lucky (or unlucky) to hit a special case, 
but we think that this is not very likely, hence our belief in the conjecture.
\end{remark}

By applying Conjectured Lemma~\ref{clemma:points} to a M\"{o}bius-general 
$6$-tuple $\vec{P}$ of base points we obtain a candidate platform~$\vec{p}$ for 
the construction of a movable hexapod. From the definition of 
the photographic map it follows that the candidate 
platform can by scaled by any non-zero factor without losing any of the 
properties ensured by Conjectured Lemma~\ref{clemma:points}. In the next 
subsection we are going to determine the right scaling factor, together 
with leg lengths, leading to a movable hexapod.

\subsection{The candidate scaling factor and leg lengths}
\label{liaison_hexapods:leg_lengths}

Here we determine the scaling factor for the $6$-tuple of platform points 
obtained in Subsection~\ref{liaison_hexapods:platform} and the leg lengths 
so that the resulting hexapods is movable. This is achieved using the following 
technique. As introduced in Subsection~\ref{degree:configuration}, it is 
possible to assign to every hexapod~$\Pi$ a projective subvariety~$K_{\Pi}$ of a 
variety~$X$ in~$\p^{16}_{\C}$, the latter being a compactification of the 
algebraic group of direct isometries of~$\R^3$. Moreover, if $K_{\Pi}$ is a 
curve we can study its intersection with a particular hyperplane~$L$ 
of~$\p^{16}_{\C}$; such intersection is constituted in general by a finite 
number of points, called bonds, contained in the intersection $B = L \cap X$, 
called the boundary. By imposing tangency conditions at the bonds between the 
curve~$K_{\Pi}$ and this hyperplane~$L$ we derive linear conditions on 
the scaling factor and on the (squares of the) leg lengths ensuring that the 
hexapod we obtain is movable.

In Subsection~\ref{liaison_hexapods:platform}, given general base 
points~$\vec{P}$, we constructed a candidate for the platform points~$\vec{p}$, 
unique up to rotations, dilations and translation. From now on, we fix one such
candidate~$\vec{p}$, and we denote by~$\gamma \tth \vec{p}$ the vector of 
points obtained by scaling~$\vec{p}$ by a non-zero factor $\gamma \in \R$.
The goal is to find a scalar~$\gamma$ and leg lengths~$\vec{d}$ such that the 
hexapod $\Pi = \bigl( \vec{P}, \gamma \tth \vec{p}, \vec{d} \bigr)$ is 
movable. 

By construction, the M\"{o}bius curves~$D$ and~$D'$ of~$\vec{P}$ and~$\gamma 
\tth \vec{p}$ intersect in $14$ points, and these points do not depend 
on~$\gamma$. As pointed out in Remark~\ref{remark:pullback}, it is not always 
the case that all intersections correspond to boundary point. However, this is 
true if we exclude the situation when both M\"{o}bius curves pass through one 
of the nodes of~$M_6$. Hence from now on we will impose another condition on 
M\"{o}bius-general $6$-tuples, namely we suppose that no three points are 
collinear. In this way --- according to Theorem~\ref{theorem:bound_degree} --- 
to each of the intersections of~$D$ and~$D'$ we can associate a boundary point 
$\beta \in B$ in the following way: if for $v, w \in C$ we have
$f_{\vec{P}}(w) = f_{\vec{p}}(v)$, then $v$ and $w$ represent directions 
in~$\R^3$ such that the orthogonal projection of~$\vec{P}$ along~$w$ and the 
orthogonal projection of~$\gamma \tth \vec{p}$ along~$v$ are M\"{o}bius 
equivalent. If $\mu$ is the M\"{o}bius transformation sending the projection 
of~$\vec{P}$ to the projection of~$\gamma \tth \vec{p}$, then by the description 
of boundary points we gave in Subsection~\ref{degree:configuration} we know that 
the triple $(v,w,\mu)$ defines a unique boundary point~$\beta_{\gamma} \in B$. 
By a direct computation one sees that as $\gamma \in \R$ varies, then 
$\beta_{\gamma}$ moves on a line contained in~$B$. By repeating this procedure 
for all~$14$ points of intersection of the two M\"{o}bius curves we get $14$ 
boundary points~$\{ \beta_{\gamma}^{k} \}$, with $k \in \{1, \dotsc, 14\}$. By 
construction, the~$\beta_{\gamma}^k$ have the property that if $\vec{d}$ are 
leg lengths such that $\Pi_{\gamma, \vec{d}} = \bigl( \vec{P}, \gamma \tth 
\vec{p}, \vec{d} \bigr)$ has no empty configuration set~$K_{\Pi_{\gamma, 
\vec{d}}}$, then $\beta_{\gamma}^{k} \in K_{\Pi_{\gamma, \vec{d}}} \cap B$ for 
all $k \in \{1, \dotsc, 14\}$.

\smallskip
We now state the tangency conditions we are interested in. To do so, we recall 
the concept of \emph{pseudo spherical condition} (already used in 
Theorem~\ref{theorem:bound_degree}), which is nothing but the 
restriction of the spherical condition in 
Eq.~\eqref{equation:spherical_condition} to the hyperplane~$L$ defining the 
boundary. In contrast with the spherical condition, the pseudo spherical 
condition does not depend on the leg lengths. The pair of $6$-tuples given 
by~$\vec{P}$ and $\gamma \tth \vec{p}$ imposes $6$ pseudo spherical conditions,
which determine a linear space that we denote by~$\tilde{H}_{\gamma}$. 
Notice that if we fix a vector~$\vec{d}$ of leg lengths,
and we create the hexapod $\Pi_{\gamma, \vec{d}} = \bigl( \vec{P}, \gamma 
\tth \vec{p}, \vec{d} \bigr)$, denoting by~$H_{\gamma, \vec{d}}$ the 
linear space cut out by the spherical conditions determined 
by~$\Pi_{\gamma, \vec{d}}$, then for all boundary points~$\beta$
\[ 
  \beta \in H_{\gamma, \vec{d}} \cap X \quad \text{if and only if} \quad 
  \beta \in \tilde{H}_{\gamma},
\]
since $H_{\gamma, \vec{d}} \cap L = \tilde{H}_{\gamma} \cap L$, where $L$ is
the hyperplane in~$\p^{16}_{\C}$ determining the boundary~$B$.
In particular this holds for all $14$ points~$\beta_{\gamma}^{k}$. 

\begin{definition} 
\label{definition:t2}
Let $\vec{P}$ and $\vec{p}$ be two $6$-tuples whose M\"{o}bius curves intersect
in~$14$ points giving rise to $14$ boundary points. Following the notation 
introduced before, denote by~$\beta_{\gamma}^{k}$ the~$14$ boundary points 
determined by~$\vec{P}$ and~$\gamma \tth \vec{p}$, and by~$\tilde{H}_{\gamma}$ 
the linear space cut out by the $6$ pseudo spherical conditions imposed 
by~$\vec{P}$ and~$\gamma \tth \vec{p}$. Suppose furthermore that 
$\tilde{H}_{\gamma}$ and $X$ intersect properly, so that each of 
the~$\beta_{\gamma}^{k}$ is an 
irreducible component of $\tilde{H}_{\gamma} \cap X$. We say that $\gamma \in \R 
\setminus \left\{ 0 \right\}$ satisfies the tangency condition~$\mathrm{Tang_2}$ 
if and only if for each of the~$14$ points~$\beta_{\gamma}^{k}$ the 
intersection multiplicity $\mathrm{i} \bigl( \tilde{H}_{\gamma}, X; 
\beta_{\gamma}^{k} \bigr)$ of~$\tilde{H}_{\gamma}$ and~$X$ 
at~$\beta_{\gamma}^{k}$ is greater than or equal to~$2$.
\end{definition}

We notice that, with the notation previously introduced, $\mathrm{i} \bigl( 
\tilde{H}_{\gamma}, X; \beta_{\gamma}^{k} \bigr) \geq 2$ if and only if 
$\mathrm{i} \bigl( H_{\gamma, \vec{d}}, X; \beta_{\gamma}^{k} \bigr) \geq 2$ for 
every~$\vec{d}$ for which $H_{\gamma, \vec{d}}$ and $X$ intersect properly. To 
prove this it is enough to show that 
\[
 \mathrm{i} \bigl( \tilde{H}_{\gamma}, X; \beta_{\gamma}^{k} \bigr) = 1 \quad
 \text{if and only if} \quad \mathrm{i} \bigl( H_{\gamma, \vec{d}}, X; 
\beta_{\gamma}^{k} \bigr) = 1.
\]
This follows from the fact that the projective tangent space $T_{\beta^{k}_{\gamma}} X$ of $X$
at~$\beta_{\gamma}^{k}$ is contained in $L$ (we used already this fact in 
Theorem~\ref{theorem:bound_degree}; this can be proved by 
a direct computation, using for example Gr\"{o}bner bases) and that
\[
  \mathrm{i} \bigl( \tilde{H}_{\gamma}, X; \beta_{\gamma}^{k} \bigr) = 1 \quad 
  \text{if and only if} \quad \mathrm{T}_{\beta^{k}_{\gamma}} X \cap \tilde{H}_{\gamma} = 
  \{ \beta_{\gamma}^{k} \}
\]
and similarly for~$H_{\gamma, \vec{d}}$. 

\begin{remark}
The condition $\mathrm{Tang_2}$ is affine-linear in~$\gamma$. In fact, pick one 
of the $14$ boundary points~$\beta_{\gamma}^{k}$; from now on we will denote it 
by~$\beta_{\gamma}$.
The condition $\mathrm{Tang_2}$ is equivalent to the condition that the dimension
of the intersection of the projective tangent space $\mathrm{T}_{\beta_{\gamma}} X$
and the linear space $\tilde{H}_{\gamma}$ is greater than 
or equal to~$1$. After possibly reparametrizing the projective line on which 
$\beta_{\gamma}$ lies as $\gamma$ varies, we can write $\beta_{\gamma} = (0 : 
\alpha w v^{T} : \lambda w : \mu v : \gamma)$, since the line passes through
the vertex of the boundary~$B$. One can show that $\mathrm{T}_{\beta_{\gamma}} X$
is spanned by the rows of the following matrix:
\begin{equation*}
  \begin{array}{ccccc}
    h & M & x & y & r \\
    \hline
    0 & w v^{T} & 0 & 0 & 0 \\
    0 & \alpha w' v^{T} & \lambda w' & 0 & 0 \\
    0 & \alpha w v'^{T} & 0 & \mu v' & 0 \\
    0 & 0 & w & 0 & 0 \\
    0 & 0 & 0 & v & 0 \\
    0 & 0 & 0 & 0 & 1 \\
    0 & 0 & w' & v' & 0
  \end{array}
\end{equation*}
where $w'$, together with~$w$, span the tangent line in~$\p^2_{\C}$ of~$C$ at~$w$,
and the same for~$v'$, subject to the condition that $\left\langle w', w' \right\rangle = 
\left\langle v', v' \right\rangle$. In particular, the projective tangent space 
does not depend on~$\gamma$. The $6$ pseudo spherical conditions defining 
$\tilde{H}_{\gamma}$ determine a linear map $\eta \colon \C^7 
\longrightarrow \C^6$, where we identify $\C^7$ with the vector space associated 
to~$\mathrm{T}_{\beta_{\gamma}} X$. Its kernel is the vector space associated to the 
intersection $\mathrm{T}_{\beta_{\gamma}} X \cap \tilde{H}_{\gamma}$.
We are going to show that the condition $\dim \ker \eta \geq 2$ is 
affine-linear in~$\gamma$. To do this, we pick coordinates so that 
(see~\cite[Subsection~2.3]{GalletNawratilSchicho1})
\begin{align*}
  & v = w = (1: i: 0), \\
  & v' = w' = (0 : 0 : 1), \\
  & \lambda = \mu = 0, \\
  & \alpha = 1.
\end{align*}
Then a direct inspection of the matrix of~$\eta$ proves the statement.
\end{remark}

\begin{example}
\label{example:georg_continued_gamma}
  Consider the pair of $6$-tuples computed in Example~\ref{example:georg}: one 
finds that each of the $14$ affine-linear equations for~$\gamma$ is a multiple 
of the equation $\gamma - 1 = 0$. 
\end{example}

\begin{definition} 
\label{definition:t3}
Let $\Pi_{\gamma, \vec{d}} = \bigl( \vec{P}, \gamma \tth \vec{p}, \vec{d} 
\bigr)$ with $\vec{P}$ and $\vec{p}$ as before, and suppose that~$\gamma$ 
satisfies~$\mathrm{Tang_2}$. Let $H_{\gamma, \vec{d}}$ be the linear 
space defined by the $6$ spherical conditions determined by~$\Pi_{\gamma, 
\vec{d}}$. Suppose that $H_{\gamma, \vec{d}}$ and~$X$ intersect properly.
We say that $\vec{d}$ satisfies the tangency condition~$\mathrm{Tang_3}$ if and 
only if for each of the~$14$ points~$\beta_{\gamma}^{k}$ the intersection 
multiplicity $\mathrm{i} \bigl( H_{\gamma, \vec{d}}, X; \beta_{\gamma}^{k} 
\bigr)$ is greater than or equal to~$3$.
\end{definition}

\begin{example}
\label{example:georg_continued_legs}
  Recall Example~\ref{example:georg_continued_gamma}. Naively one expects that 
  the intersection $\mathrm{i} 
  \bigl( H_{\gamma, \vec{d}}, X; \beta_{\gamma}^{k} \bigr)$ is greater than
  or equal to~$3$ if we are able to find a solution of the system of equations
  given by $H_{\gamma, \vec{d}} \cap X$ in $\C[t]/(t^3)$. For a solution of 
  the form $c_0 + c_1 t + c_2 t^2$, the coefficients~$c_0$ and~$c_1$ are determined
  by $\beta_{\gamma}^{k}$ itself and by a tangent vector in $\mathrm{T}_{\beta_{\gamma}^{k}}
  \cap H_{\gamma, \vec{d}}$, which is unique up to scaling. For~$c_2$, we obtain a
  system of affine-linear equations in $d_1^2, \dotsc, d_6^2$ and~$c_2$. The 
  solvability of these equations with respect to~$c_2$ is equivalent to another system
  of affine-linear equations in $d_1^2, \dotsc, d_6^2$. These equations are:
  \begin{align}
   \label{equation:example_legs1}
    d_4^2&=\tfrac{71}{92}d_1^2-\tfrac{105}{92}d_2^2+\tfrac{63}{46}d_3^2-\tfrac{535801}{676062}, \\ 
    \label{equation:example_legs2}
    d_5^2&=\tfrac{71}{41}d_1^2-\tfrac{75}{41}d_2^2+\tfrac{45}{41}d_3^2-\tfrac{1908080}{1074159},\\ 
    \label{equation:example_legs3}
    d_6^2&=\tfrac{71}{44}d_1^2-\tfrac{45}{44}d_2^2+\tfrac{9}{22}d_3^2-\tfrac{114265}{154638}.
  \end{align}
\end{example}

\begin{theorem} 
\label{thm:movability}
Assume that $\vec{P}$ is a $6$-tuple of points in~$\R^3$. Assume that 
$\vec{p}$ is another 6-tuple such that the M\"{o}bius curves of~$\vec{P}$ and 
of~$\vec{p}$ intersect in~$14$ points and do not intersect in a node of~$M_6$. 
Assume that $\gamma$ and $\vec{d}$ satisfy the conditions~$\mathrm{Tang_2}$ 
and~$\mathrm{Tang_3}$. Then the hexapod $\Pi = \bigl( \vec{P}, \gamma \tth 
\vec{p}, \vec{d} \bigr)$ is movable.
\end{theorem}

\begin{proof}
Since the two M\"{o}bius curves do not intersect in a node, then from the 
discussion at the beginning of the section we obtain $14$ boundary points. 
As before, denote by~$H_{\Pi}$ the linear space defined by the $6$ spherical 
conditions determined by~$\Pi$. Suppose that $H_{\Pi}$ and~$X$ intersect 
properly in finite number of points, namely that~$\Pi$ is not movable. By 
assumption, the intersection multiplicity of~$X$ and~$H_\Pi$ at each of 
the~$14$ bonds of~$\Pi$ corresponding to the~$14$ common images of the 
M\"{o}bius curves is at least~$3$. Hence the intersection count 
gives $14 \cdot 3 > 40 = \left( \deg X \right) \left( \deg H \right)$, which is 
absurd by~\cite[Appendix~A, Axiom~A6]{Hartshorne1977}. Therefore $\Pi$ is 
movable.
\end{proof}

\begin{remark}
\label{remark:liaison_hexapods}
  If $\vec{P}$ is a M\"{o}bius-general $6$-tuple of points in~$\R^3$, then 
Conjectured Lemma~\ref{clemma:points} predicts the existence of a $6$-tuple 
$\vec{p}$ satisfying the hypothesis of Theorem~\ref{thm:movability}. 
\end{remark}

\begin{conjlemma}
\label{clemma:legs}
	For a M\"{o}bius-general $6$-tuple $\vec{P}$ of base points, consider the 
$6$-tuple $\vec{p}$ of platform points predicted by Conjectured 
Lemma~\ref{clemma:points}. There is a unique~$\gamma$ such 
that the tangency condition~$\mathrm{Tang_2}$ is fulfilled.
Moreover, the set of leg length vectors~$\vec{d}$ fulfilling the tangency 
condition~$\mathrm{Tang_3}$ is of dimension~$3$.
\end{conjlemma}

\begin{definition}
\label{definition:liaison_hexapods}
Let $\vec{P}$ be a M\"{o}bius-general $6$-tuple of points in~$\R^3$. Then 
Conjectured Lemma~\ref{clemma:points} and Conjectured Lemma~\ref{clemma:legs} 
predict the existence of~$\vec{p}$, $\gamma$ and~$\vec{d}$ satisfying the 
hypothesis of Theorem~\ref{thm:movability}. We call the resulting hexapod~$\Pi$ 
a \emph{liaison hexapod}.
\end{definition}

As for Conjectured Lemma~\ref{clemma:points}, Conjectured 
Lemma~\ref{clemma:legs} is formulated such a way that if it is false, then it is 
falsified by a generic choice of base points. We tested the conjecture against 
random numerical examples, and constructed many movables hexapods in this way. 
But, as in Remark~\ref{remark:faith}, we could have been lucky (or unlucky) in 
all these experiments and the conjecture may be false.

Finally, we would like to point out that, in case the two Conjectured Lemmata
hold, the family of liaison hexapods is maximal among movable hexapods. In fact,
the possible bases of such hexapods form an open subset in~$\left( \R^3 \right)^6$;
moreover, if the conjectures are true, there exists exactly a $3$-dimensional
collection (parametrized by the leg lengths) of movable hexapods of maximal 
conformal degree~$28$ for each such base. This family cannot be contained in a 
larger family with smaller conformal degree since the latter is upper-semicontinuous.

\section{Computations in Study parameter space}
\label{study}

In this section we exhibit two positive-dimensional families of base 
points~$\vec{P}$ for which we can explicitly compute the candidate platform 
points~$\vec{p}$ coming from the liaison procedure in concrete instances as  
explained in Subsection~\ref{liaison_hexapods:platform}, and such that there 
exists a unique scaling factor~$\gamma$ satisfying the 
condition~$\mathrm{Tang_2}$ and a three-dimensional set of leg lengths~$\vec{d}$ 
satisfying the condition~$\mathrm{Tang_3}$. We show that the hexapods 
corresponding to such families are movable.

\begin{proposition}
\label{proposition:family_lines}
	Consider a $6$-tuple $\vec{P}$ of points in~$\R^3$ such that the 
lines $\overrightarrow{P_1P_2}$, $\overrightarrow{P_3P_4}$ and 
$\overrightarrow{P_5P_6}$ meet in one point (see 
Fig.~\ref{figure:concurrent_lines}). Then consider the following setting:
	\begin{itemize}
	 \item[i.] Take the candidate platform~$\vec{p}$ to be:
		\[ (p_1,p_2,p_3,p_4,p_5,p_6) \; = \; (P_2, P_1, P_4, P_3, P_6, P_5). \]
	 \item[ii.] The scaling factor~$\gamma$ is given by~$-1$.
	 \item[iii.] The condition on the leg lengths~$\{ d_i \}$ ensuring 
mobility one is:
		\[ d_1^2 = d_2^2 \qquad d_3^2 = d_4^2 \qquad d_5^2 = d_6^6. \]
	\end{itemize}
	Then the hexapod that we obtain is movable.
\end{proposition}

\begin{proposition}
\label{proposition:family_order3}
	Consider a $6$-tuple $\vec{P}$ of points in~$\R^3$ such that there exists 
an isometry~$\sigma$ of~$\R^3$ of order~$3$ acting on~$\vec{P}$ in the 
following way:
	\[ \sigma \colon \, (P_1, P_2, P_3, P_4, P_5, P_6) \; \mapsto \; (P_2, P_3, 
P_1, P_5, P_6, P_4). \]
	Then consider the following setting:
	\begin{itemize}
	 \item[i.] Take the candidate platform~$\vec{p}$ to be:
		\[ (p_1,p_2,p_3,p_4,p_5,p_6) \; = \; (P_4, P_6, P_5, P_1, P_3, P_2). \]
	 \item[ii.] The scaling factor~$\gamma$ is given by~$1$.
	 \item[iii.] The condition on the leg lengths~$\{ d_i \}$ ensuring 
mobility one is given by a system linear in the~$d_i^2$ admitting a 
three-dimensional solution set. 
	\end{itemize}
	Then the hexapod that we obtain is movable.
\end{proposition}
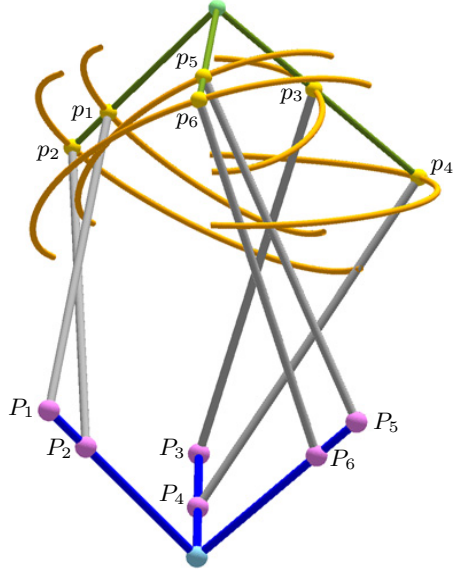
\begin{figure}[!ht]
	\begin{center}
	\begin{overpic}[width=55mm]{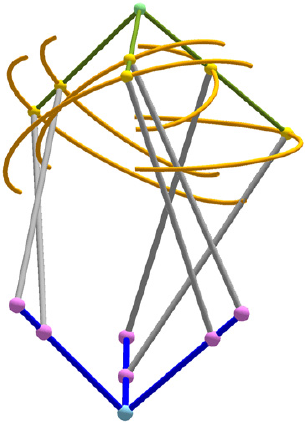}
		\begin{small}
			\put(-3,27){$P_1$}
			\put(3.5,20){$P_2$}
			\put(23,20){$P_3$}
			\put(23,12){$P_4$}
			\put(60,24.5){$P_5$}
			\put(52.5,18){$P_6$}
			\put(8,79){$p_1$}
			\put(2,72){$p_2$}
			\put(44,82){$p_3$}
			\put(70,69.5){$p_4$}
			\put(26,88){$p_5$}
			\put(26,77.5){$p_6$}
		\end{small}
	\end{overpic}
	\end{center}
	\caption{Illustration of a self-motion of a
	hexapod belonging to the family given in
	Proposition~\ref{proposition:family_lines}: the lines 
	$\protect\overrightarrow{P_1P_2}$, $\protect\overrightarrow{P_3P_4}$ and 
	$\protect\overrightarrow{P_5P_6}$ intersect in a point, and the same 
	happens for the lines $\protect\overrightarrow{p_1p_2}$, 
	$\protect\overrightarrow{p_3p_4}$ and $\protect\overrightarrow{p_5p_6}$.
	The balls labeled by $p_i$ and $P_i$, for $i \in \{1,\dots, 6\}$, are 
	spherical joints.}
	\label{figure:concurrent_lines}
\end{figure}
We are going to prove Proposition~\ref{proposition:family_lines} and 
Proposition~\ref{proposition:family_order3} using Study parameters.
Moreover we report some interesting properties observed within this approach, 
and finally we compute the configuration curve of an example of a 
liaison hexapod.

\subsection{Basics}
 
Due to a result of Husty (see~\cite{Husty}), we use Study 
parameters $(e_0:e_1:e_2:e_3:f_0:f_1:f_2:f_3)$ to compute  
the configurations of a hexapod --- this is also known as the forward 
kinematics problem. Note that the first four homogeneous coordinates 
$(e_0:e_1:e_2:e_3)$ are the Euler parameters already mentioned in the Introduction. 
All real points of the Study parameter space~$\p^7_{\C}$ that are located on the 
so-called Study quadric --- given by $\Psi=0$, where $\Psi = \sum_{i=0}^3e_i 
\tth f_i$ --- correspond to a direct isometry, with exception of the 
3-dimensional subspace $e_0=e_1=e_2=e_3=0$, as its points do not fulfill the 
condition $N \neq 0$ with $N=e_0^2+e_1^2+e_2^2+e_3^2$. The translation vector 
$(t_1,t_2,t_3)$ and the rotation matrix~$R$ of the 
direct isometry $(x,y,z) \mapsto (x,y,z) R  + (t_1,t_2,t_3)$ corresponding to a 
point in the Study quadric are given by:
\begin{equation*}
\begin{split}
	t_1 &= 2(e_0f_1-e_1f_0+e_2f_3-e_3f_2), \quad
	t_2  = 2(e_0f_2-e_2f_0+e_3f_1-e_1f_3), \\
	t_3 &= 2(e_0f_3-e_3f_0+e_1f_2-e_2f_1),
\end{split}
\end{equation*}
and
\begin{equation*}
  R = \begin{pmatrix} 
    e_0^2+e_1^2-e_2^2-e_3^2 & 2(e_1e_2+e_0e_3) & 2(e_1e_3-e_0e_2) \\
    2(e_1e_2-e_0e_3) & e_0^2-e_1^2+e_2^2-e_3^2 & 2(e_2e_3+e_0e_1) \\
    2(e_1e_3+e_0e_2) & 2(e_2e_3-e_0e_1) & e_0^2-e_1^2-e_2^2+e_3^2
\end{pmatrix},
\end{equation*}
if the normalizing condition $N = 1$ is fulfilled.

By using the Study parametrization of direct isometries, the condition 
that the point~$p_i$ is located on a sphere centered in~$P_i$ with radius~$d_i$ 
is a quadratic homogeneous equation according to Husty (see~\cite{Husty}).
For the explicit formula of the used spherical condition~$\Lambda_i=0$ we refer 
to~\cite[Eq.~(2)]{Nawratil2014}.

The solution for the direct kinematics over~$\C$ of a hexapod can be 
written as the algebraic variety whose ideal is spanned 
by~$\Psi$, $\Lambda_1, \ldots, \Lambda_6$, with the condition $N = 1$.
In general this variety consists of a discrete set of points, corresponding
to the (at most) $40$ solutions of the forward kinematic problem. In the case of 
liaison linkages it is $1$-dimensional as they have a self-motion.

\subsection{Proving the two propositions}

We consider the polynomials $\Delta_{i,j} := \Lambda_i - \Lambda_j$ for $i<j$ and 
$i,j \in \left\{ 1, \ldots, 6 \right\}$, which are affine-linear in $f_0, 
\ldots, f_3$ and therefore they can be written in the form:
\begin{equation*}
	\Delta_{i,j} \; = \;
	S_{i,j} \tth f_0 + T_{i,j} \tth f_1 + U_{i,j} \tth f_2 + V_{i,j} \tth f_3 + 
W_{i,j}.
\end{equation*}
The system of polynomials~$\Delta_{i,j}$ can be grouped into the following six 
sets:
\begin{equation*}
\mathcal{S}_k \; := \; \bigl\{
\Delta_{i,j} \, : \, i<j \quad \text{and} \quad i,j \in \left\{ 1, 
\ldots, 6 \right\} \setminus \left\{ k \right\}
\bigr\} \quad \text{for} \quad k = 1, \ldots, 6.
\end{equation*}
Denote by~$I_k$ the ideal generated by the set~$\left\{ \mathcal{S}_k \right\}$, together with 
the polynomial~$\Psi$. It can easily be seen that each 
of these ideals is generated by five linear polynomials in $f_0,\ldots ,f_3$:
\begin{align*}
	I_6 & = \left\langle 
	\Delta_{1,2},\Delta_{1,3},\Delta_{1,4},\Delta_{1,5},\Psi \right\rangle, &\quad
	I_5 & = \left\langle 
	\Delta_{1,2},\Delta_{1,3},\Delta_{1,4},\Delta_{1,6},\Psi \right\rangle, \\
	I_4 & = \left\langle 
	\Delta_{1,2},\Delta_{1,3},\Delta_{1,5},\Delta_{1,6},\Psi \right\rangle, &\quad
	I_3 & = \left\langle 
	\Delta_{1,2},\Delta_{1,4},\Delta_{1,5},\Delta_{1,6},\Psi \right\rangle, \\
	I_2 & = \left\langle 
	\Delta_{1,3},\Delta_{1,4},\Delta_{1,5},\Delta_{1,6},\Psi \right\rangle, &\quad
	I_1 & = \left\langle 
	\Delta_{2,3},\Delta_{2,4},\Delta_{2,5},\Delta_{2,6},\Psi \right\rangle.
\end{align*}
A necessary condition for the solvability of a system of five linear equations in $f_0,\ldots ,f_3$
is that the determinant of the extended coefficient matrix is equal to zero. 
This condition is denoted by~$\Omega_k=0$; e.g.\ $\Omega_6$ is given by:
\begin{equation*}
	\Omega_6: \quad
	\begin{vmatrix}
		S_{1,2} & T_{1,2} & U_{1,2} & V_{1,2} & W_{1,2} \\
		S_{1,3} & T_{1,3} & U_{1,3} & V_{1,3} & W_{1,3} \\
		S_{1,4} & T_{1,4} & U_{1,4} & V_{1,4} & W_{1,4} \\
		S_{1,5} & T_{1,5} & U_{1,5} & V_{1,5} & W_{1,5} \\
		e_0 & e_1 & e_2 & e_3 & 0 
	\end{vmatrix}
	.
\end{equation*}
It is well known (see~\cite{Husty}) that $\Omega_k$ factors into 
$e_0^2+e_1^2+e_2^2+e_3^2$ and a quartic factor $G_k$ that has $258\, 720$ 
terms. In this way we get six quartic equations $G_k=0$ in the Euler 
parameter space. Now it can easily be checked by direct computations 
that 
\begin{equation*}
	G_1 - G_2 + G_3 - G_4 + G_5 - G_6 \; = \; 0
\end{equation*}
holds, i.e.\ the~$\{ G_i \}$ are linearly dependent.
Therefore we can restrict to the equations $G_2=0, \ldots ,G_6=0$. Based on this 
preparatory work we prove the mobility of two classes of liaison examples.

\subsubsection{Proof of Proposition~\ref{proposition:family_lines}}

We assume that two Cartesian frames --- called the moving and the fixed frame 
--- are rigidly attached to the platform and the base of the hexapods, respectively.
Without loss of generality we can choose these frames in a way that 
the base and platform points have the following coordinates 
(with respect to the corresponding frames):
\begin{align*}
P_1&=(A_1,0,0), &\quad P_2&=\mu_1 P_1, &\quad p_1&=-P_2, &\quad p_2&=-P_1, \\
P_3&=(A_3,B_3,0), &\quad P_4&=\mu_3 P_3, &\quad p_3&=-P_4, &\quad p_4&=-P_3, \\
P_5&=(A_5,B_5,C_5), &\quad P_6&=\mu_5 P_5, &\quad p_5&=-P_6, &\quad p_6&=-P_5. 
\end{align*}
Plugging these coordinates and leg lengths relations $d_1=d_2$, $d_3=d_4$ and $d_5=d_6$ 
into our above calculated expressions 
shows that $G_k$ factors into a linear expression~$L_k$ in the Euler parameters 
and a common cubic factor~$S$ with~$650$ terms.\footnote{The corresponding Maple Worksheet 
can be downloaded as \texttt{mws} file and \texttt{pdf} file from 
\url{www.geometrie.tuwien.ac.at/nawratil/prooffamily1.mws} and
\url{www.geometrie.tuwien.ac.at/nawratil/prooffamily1.pdf}, respectively.}
This already proves that the mobility is one, because the dimension of the 
variety of the ideal $I = I_2 + I_3 + I_4 + I_5 + I_6$ is at least~$2$, as the 
latter is generated by:
\begin{equation*}
	I \; = 
\; \left\langle \Delta_{1,2}, \Delta_{1,3}, \Delta_{1,4}, 
\Delta_{1,5}, \Delta_{1,6}, \Psi \right\rangle.
\end{equation*}
Since there is only one equation left (any equation~$\Lambda_i=0$ can be taken) 
we get at least mobility one over~$\C$. 

\begin{remark}
Note that $S=0$ can split up into several components. An example for this is 
the case $\mu_i=\mu_j=-1\neq \mu_k$ with pairwise distinct 
$i,j,k \in \left\{ 1,3,5 \right\}$. In this case the cubic surface $S=0$ splits 
up into a quadric and a plane. For $i=1$, $j=3$ and the additional relations 
$A_3=0$, $B_3=A_1$ and $d_1=d_3$ we get even 3 planes, namely:
\begin{equation*}
	e_1-e_2=0, \quad 
	e_1+e_2=0, \quad
	A_5e_1+B_5e_2+C_5e_3=0.
\end{equation*}
\end{remark}

\subsubsection{Proof of Proposition~\ref{proposition:family_order3}}
Without loss of generality we can choose the fixed frame and the moving frame 
in a way that the base and platform points have the following coordinates 
(with respect to the corresponding frames):
\begin{align*}
p_1&=(a,b,c), &\quad p_4&=(A,B,C), &\quad P_1&= p_4, &\quad P_4&= p_1, \\
p_2&=(b,c,a), &\quad p_5&=(B,C,A), &\quad P_2&= p_6, &\quad P_5&= p_3, \\
p_3&=(c,a,b), &\quad p_6&=(C,A,B), &\quad P_3&= p_5, &\quad P_6&= p_2. 
\end{align*}
Moreover with respect to this choice of coordinates the leg lengths can be 
expressed as 
\begin{equation*}
\begin{split}
d_4^2&=-\frac{UW(d_2^2-d_1^2)-VW(d_3^2-d_1^2)}{kK}+d_1^2, \\
d_5^2&=+\frac{UV(d_2^2-d_1^2)-UW(d_3^2-d_1^2)}{kK}+d_1^2, \\
d_6^2&=+\frac{VW(d_2^2-d_1^2)+UV(d_3^2-d_1^2)}{kK}+d_1^2, 
\end{split}
\end{equation*}
with
\begin{align*}
U&=Aa-Ab+Bb-Bc-Ca+Cc, &\quad K&= A^2+B^2+C^2-AB-AC-BC, \\
V&=Aa-Ac-Ba+Bb-Cb+Cc, &\quad k&= a^2+b^2+c^2-ab-ac-bc, \\
W&=Ab-Ac-Ba+Bc+Ca-Cb. 
\end{align*}
With respect to these coordinates and leg lengths the numerator of
$G_k$ splits up into the factor\footnote{Note that for $k=K$ platform 
and base are congruent.} $k-K$, a linear expression~$L_k$ in the Euler 
parameters and a common cubic factor~$S$ with $576$ terms.\footnote{The 
corresponding Maple Worksheet can be downloaded as \texttt{mws} file and 
\texttt{pdf} file from 
\url{www.geometrie.tuwien.ac.at/nawratil/prooffamily2.mws} and
\url{www.geometrie.tuwien.ac.at/nawratil/prooffamily2.pdf}, respectively.} 
This proves that the mobility is one for the same reasons as in the last proof. 

\begin{remark}
Finally it should be mentioned that the two families given in 
Propositions~\ref{proposition:family_lines} and~\ref{proposition:family_order3} 
also contain geometries that are excluded from the conjectures formulated in 
Section~\ref{liaison_hexapods} (e.g.\ planar or congruent 
hexapods, or hexapods with $4$ collinear points). 
\end{remark}

\subsection{Observations}\label{observe}

Based on random examples 
(see Section~\ref{study:generic}) and the families 
given in Propositions~\ref{proposition:family_lines} 
and~\ref{proposition:family_order3} we provide the following interesting 
observations, which hopefully could lead to a simpler (or even explicit) 
computation of the linked $6$-tuple in the future.

\begin{enumerate}[I.]
\item
Each $G_k$ factors in a linear expression~$L_k$ and a common cubic 
factor~$S$ for $k=2,\ldots,6$. Every plane defined by a linear 
equation $L_k=0$ belongs to a bundle of planes with vertex~$V$ in the Euler 
parameter space. Therefore there exists a $2$-parametric set of linear 
combinations:
\begin{equation*}
\gamma_2 \tth G_2 + \gamma_3 \tth G_3 + \gamma_4 \tth G_4 + \gamma_5 \tth G_5 + 
\gamma_6 G_6 \; = \; 0.
\end{equation*}
Moreover, the vertex~$V$ belongs to the common cubic surface~$S=0$. 
\item
The vertex~$V$ does not depend on the remaining leg lengths. There exists a 
bijection 
between the remaining three leg lengths and the translation vector; i.e.\ a 
self-motion can be started from every pose of the platform, if it has this 
special orientation. 

\begin{remark}
Therefore this manipulator can be seen as an translational singular 
manipulator (or Cartesian-singular manipulators; see 
\cite[Section~5]{nawratil_schoenflies}) with respect to this orientation.
\end{remark}
\item
Moreover $V$ corresponds to an orientation of the manipulator, where 
the five difference vectors
\begin{equation}
\label{equation:diff_vec}
(P_i-P_1)-(p_i-p_1) \quad\text{for}\quad i=2,\ldots ,6
\end{equation}
only span a plane~$\alpha$ with normal vector~$n$.

\noindent Now we consider the orthogonal projection of the points $p_1, \ldots, 
p_6$ and $P_1, \ldots, P_6$ on~$\alpha$ which yields $p_1^{\prime}, \ldots, 
p_6^{\prime}$, $P_1^{\prime}, \ldots, P_6^{\prime}$.
There exist three pairs\footnote{Note that not all three pairs have to be real 
as two pairs can also be complex conjugate.} of centers $(q_i, Q_i)$ 
in~$\alpha$, in a way that the two pencils of six lines
\begin{equation}
\label{equation:pencils}
	[q_i,p_1^{\prime}], \ldots ,[q_i,p_6^{\prime}] \quad \text{and} \quad
	[Q_i,P_1^{\prime}], \ldots ,[Q_i,P_6^{\prime}]
\end{equation}
can be mapped onto each other by a congruence sending 
$[q_i,p_j^{\prime}]$ to $[Q_i,P_j^{\prime}]$ for $j=1,\ldots ,6$.
Moreover there is an orientation reversing equiform transformation with 
$q_i \mapsto Q_i$ for $i=1,2,3$. Observation III is illustrated in 
Fig.~\ref{fig0} with respect to the example given in 
Section~\ref{study:generic}. 
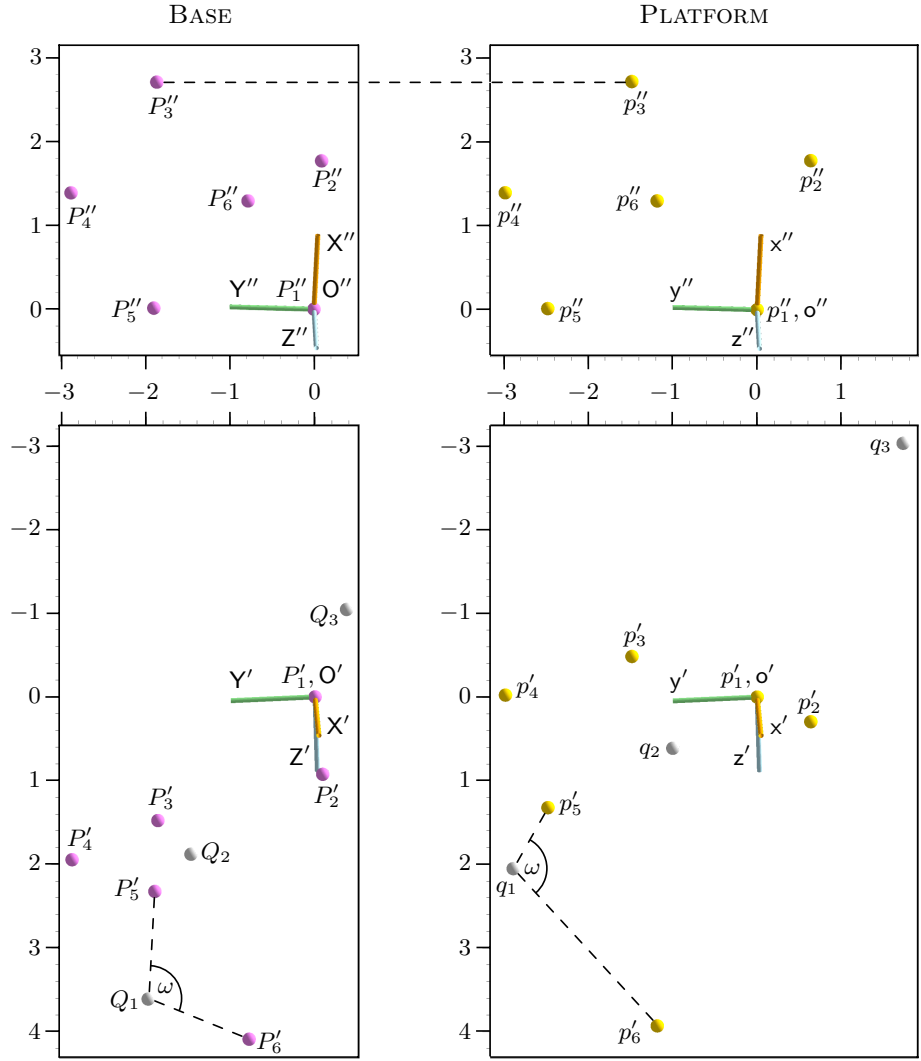
\begin{figure}[top] 
$\phm$ \vspace{7mm} \\
\hfill
 \begin{overpic}
    [width=115mm]{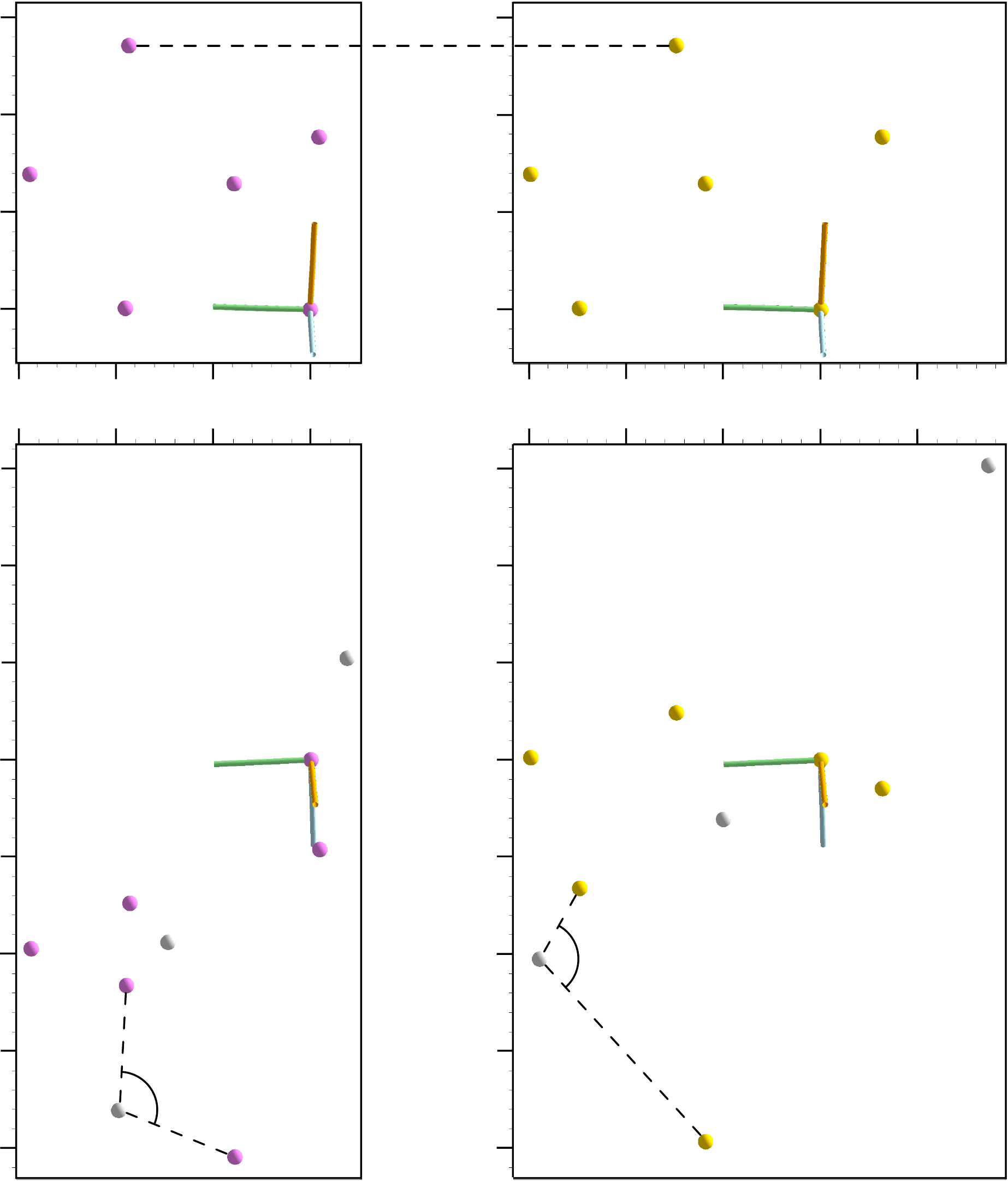}
\put(12,102){{\sc Base}}
\put(58,102){{\sc Platform}}
\begin{small}
\put(2.1,21){$P_4^{\prime}$}
\put(6.7,16.1){$P_5^{\prime}$}
\put(10,25){$P_3^{\prime}$}
\put(15.2,19.7){$Q_2$}
\put(6.3,5.2){$Q_1$}
\put(11.0,6.5){$\omega$}
\put(20.6,1){$P_6^{\prime}$}
\put(26.2,25.2){$P_2^{\prime}$}
\put(23,37){$P_1^{\prime},\go O^{\prime}$}
\put(25.8,43){$Q_3$}
\put(27.5,32){$\go X^{\prime}$}
\put(23.8,29){$\go Z^{\prime}$}
\put(18,36.5){$\go Y^{\prime}$}
\put(56,2){$p_6^{\prime}$}
\put(44,16.5){$q_1$}
\put(46.8,18.2){$\omega$}
\put(50.2,24.5){$p_5^{\prime}$}
\put(46,36){$p_4^{\prime}$}
\put(56.5,41.1){$p_3^{\prime}$}
\put(66,37){$p_1^{\prime},\go o^{\prime}$}
\put(73.5,34.5){$p_2^{\prime}$}
\put(58,30){$q_2$}
\put(80.5,60){$q_3$}
\put(70.8,32){$\go x^{\prime}$}
\put(67.2,29){$\go z^{\prime}$}
\put(61,36.5){$\go y^{\prime}$}
\put(-0.2,65){$-3$}
\put(8,65){$-2$}
\put(16,65){$-1$}
\put(25.7,65){$0$}
\put(43,65){$-3$}
\put(51,65){$-2$}
\put(59,65){$-1$}
\put(68.7,65){$0$}
\put(77.1,65){$1$}
\put(2,82.2){$P_4^{\prime\prime}$}
\put(6.4,73){$P_5^{\prime\prime}$}
\put(10,93){$P_3^{\prime\prime}$}
\put(15.8,84){$P_6^{\prime\prime}$}
\put(26,85.8){$P_2^{\prime\prime}$}
\put(22.5,75){$P_1^{\prime\prime}$}
\put(27,75){$\go O^{\prime\prime}$}
\put(27.5,79.5){$\go X^{\prime\prime}$}
\put(23,70.2){$\go Z^{\prime\prime}$}
\put(18,75){$\go Y^{\prime\prime}$}
\put(70.8,79.5){$\go x^{\prime\prime}$}
\put(67,70.2){$\go z^{\prime\prime}$}
\put(61,75){$\go y^{\prime\prime}$}
\put(55.8,84){$p_6^{\prime\prime}$}
\put(50.2,73){$p_5^{\prime\prime}$}
\put(44.2,82.4){$p_4^{\prime\prime}$}
\put(56.5,93.4){$p_3^{\prime\prime}$}
\put(70.5,73){$p_1^{\prime\prime},\go o^{\prime\prime}$}
\put(73.7,85.8){$p_2^{\prime\prime}$}
\put(-1.5,73){$0$}
\put(-1.5,81.3){$1$}
\put(-1.5,89.6){$2$}
\put(-1.5,97.9){$3$}
\put(40.5,73){$0$}
\put(40.5,81.3){$1$} 
\put(40.5,89.6){$2$}
\put(40.5,97.9){$3$}
\put(-3.5,2){$\phm 4$}
\put(-3.5,10.3){$\phm 3$}
\put(-3.5,18.6){$\phm 2$}
\put(-3.5,26.8){$\phm 1$}
\put(-3.5,35.1){$\phm 0$}
\put(-3.5,43.3){$- 1$}
\put(-3.5,51.6){$- 2$}
\put(-3.5,59.7){$- 3$}
\put(38.5,2){$\phm 4$}
\put(38.5,10.3){$\phm 3$}
\put(38.5,18.6){$\phm 2$}
\put(38.5,26.8){$\phm 1$}
\put(38.5,35.1){$\phm 0$}
\put(38.5,43.3){$- 1$}
\put(38.5,51.6){$- 2$}
\put(38.5,59.7){$- 3$}
\end{small}
  \end{overpic} 
\caption{
We illustrate observations III by an orthogonal projection onto~$\alpha$ (top 
view indicated by~$^{\prime}$). In the corresponding front view (indicated by 
$^{\prime\prime}$) the property of Eq.~\eqref{equation:diff_vec} can be 
seen, as $p_i^{\prime\prime}$ and $P_i^{\prime\prime}$ are located on 
horizontal lines; e.g.\ dashed line $[p_3^{\prime\prime},P_3^{\prime\prime}]$. 
Moreover it can be figured out that the triangles 
$Q_1,Q_2,Q_3$ and $q_1,q_2,q_3$ are reflection similar. The pencils of lines 
given in Eq.~\eqref{equation:pencils} are not drawn as otherwise the figure 
gets overloaded. 
But the reader can verify the congruence of the corresponding pencils by 
checking the measurements with a protractor; e.g.\ 
$\omega:=\sphericalangle(P_5^{\prime},Q_1,P_6^{\prime})=\sphericalangle(p_5^{\prime},q_1,p_6^{\prime})$.
}
\label{fig0}
\end{figure} 
\item
The cubic surface $S=0$ in the Euler parameter space 
contains a line $\ell$ through the vertex $V$. The points of these line~$\ell$ 
correspond to the rotation of the platform about a fixed line orthogonal to~$\alpha$. 
In each configuration of the $2$-parametric set, obtained from the composition 
of this rotation and a translation of the platform in direction of~$n$, a 
self-motion can be started. 

\begin{remark}
\label{remark:schoenflies}
The planar hexapod with platform $p_1^{\prime},\ldots ,p_6^{\prime}$ and base
$P_1^{\prime},\ldots ,P_6^{\prime}$ is even Sch{\"onflies}-singular 
(see~\cite{nawratil_schoenflies}) with respect to the direction orthogonal to 
the $\alpha$-parallel carrier planes of the planar platform and planar base.
\end{remark}
\item
There exists a regular projectivity mapping $P_i$ to $p_i$ for $i=1, \ldots, 
6$. 
\end{enumerate}

\subsection{Example}
\label{study:generic}
We choose the following set of base points with respect to the fixed frame 
$(\go O;\go X,\go Y,\go Z)$: $P_i=A_i$ for $i=1, \ldots ,6$ 
with $A_i$ given in Eqs.~\eqref{A123} and~\eqref{A456}. 

According to Example \ref{example:georg} the liaison technique explained 
in Section~\ref{liaison_hexapods} yields the following platform with 
respect to the moving frame $(\go o;\go x,\go y,\go z)$: 
$p_i=B_i$ for $i=1, \ldots ,6$ 
with $B_i$ given in Eqs.~\eqref{B12} to \eqref{B56}.
Note that this coordinatization of the platform already 
corresponds with the special orientation~$V$.  

Moreover the leg lengths are given by Eqs.~\eqref{equation:example_legs1} to 
\eqref{equation:example_legs3}.

For the computation of the self-motion we express $f_0,f_1,f_2,f_3$ from 
$\Delta_{i,j}=0$, $\Delta_{i,k}=0$, $\Delta_{i,l}=0$, $\Psi=0$ for pairwise 
distinct $i,j,k,l\in\left\{1,\ldots ,6\right\}$ and insert them 
into~$\Lambda_i$. We denote the numerator of the resulting expression 
by~$E_{m,n}$ with pairwise distinct $i,j,k,l,m,n \in \left\{1,\ldots 
,6\right\}$. This is an octic expression in the Euler parameter space, where 
$e_0$ appears maximally to the power of~$6$. 

Due to the involved powers of~$e_0$, we eliminate this Euler parameter by 
computing the resultant of~$S$ and~$E_{m,n}$, which yields the expression 
$F_{m,n}$ of degree~$22$ in $e_1,e_2,e_3$. The greatest common divisor~$J$ of 
all~$F_{m,n}$ corresponds to the self-motion, which is in the generic case of 
degree~$12$. 

It is always possible to choose special values for the remaining leg lengths 
$d_1,d_2,d_3$ in a way that $S$ is linear in~$e_0$ and that all~$E_{m,n}$ are 
of degree~$5$ in~$e_0$. For our example the overdetermined system of equations 
resulting from the coefficients of $e_0^2e_1$, $e_0^2e_2$, $e_0^2e_3$ of~$S$ and 
the coefficients of $e_0^6e_1^2$, $e_0^6e_2^2$, $e_0^6e_3^2$, $e_0^6e_1e_2$, 
$e_0^6e_1e_3$, $e_0^6e_2e_3$ of~$E_{m,n}$ has the following solution:
\begin{figure}[t] 
	\begin{center} 
	\begin{overpic}
			[width=120mm]{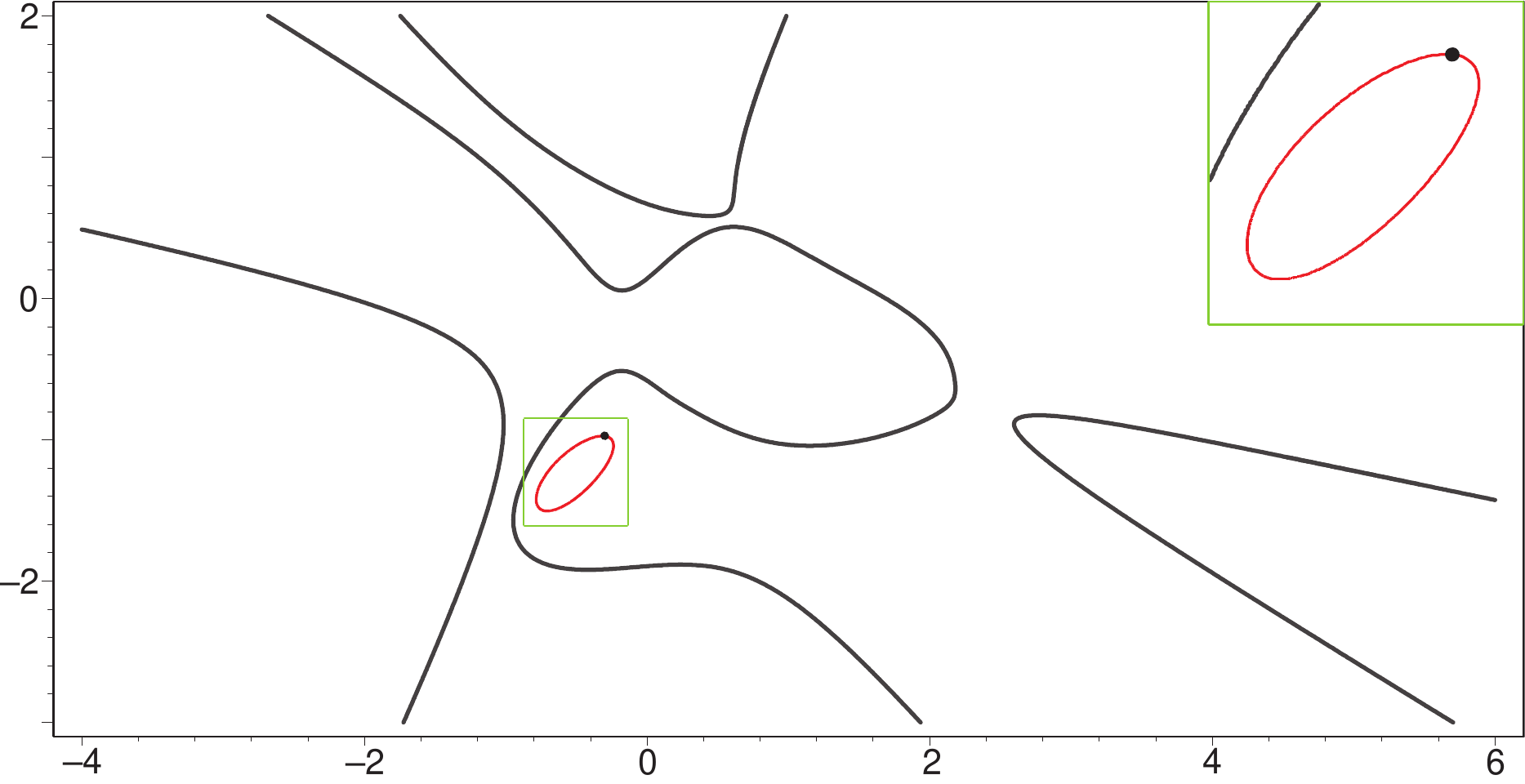}
		\end{overpic} 
	\end{center}
	\caption{
	We identify $e_3=0$ with the line at infinity and illustrate the affine part 
of the decic $J=0$, i.e.\ we set $e_3=1$ and plot $e_1$ horizontally and $e_2$ 
vertically. Note that the complete decic corresponds to a real self-motion as 
$S$ depends only linearly on~$e_0$. Moreover in the upper right corner we 
provided a zoom of the red component by a scaling factor of~$3$. 
}
	\label{fig1}
\end{figure}
\begin{equation*}
	d_1^2=\tfrac{62434791769}{2888740009},\quad
	d_2^2=\tfrac{147143743}{8595735},\quad
	d_3^2=\tfrac{431695696}{46416969}.
\end{equation*}
In this case $J$ is only of degree~$10$. This decic $J=0$ is illustrated in 
Fig.~\ref{fig1} and it consists of two components, a black and a red colored 
one. On the red component a point is highlighted, whose corresponding 
configuration is illustrated in Fig.~\ref{fig2} together with the 
trajectories corresponding to the red component of the self-motion.

\begin{remark}
Finally it should be noted that the number of~$14$ bonds can also be verified 
within the Study parameter approach according to the method presented 
in~\cite{Nawratil2014}.
\end{remark}

\section*{Acknowledgments}

The first-named and third-named author's research is supported by the Austrian 
Science Fund (FWF): W1214-N15/DK9 and P26607 - ``Algebraic Methods in 
Kinematics: Motion Factorisation and Bond Theory''. The second-named author's 
research is funded by the Austrian Science Fund (FWF): P24927-N25 - ``Stewart 
Gough platforms with self-motions''.

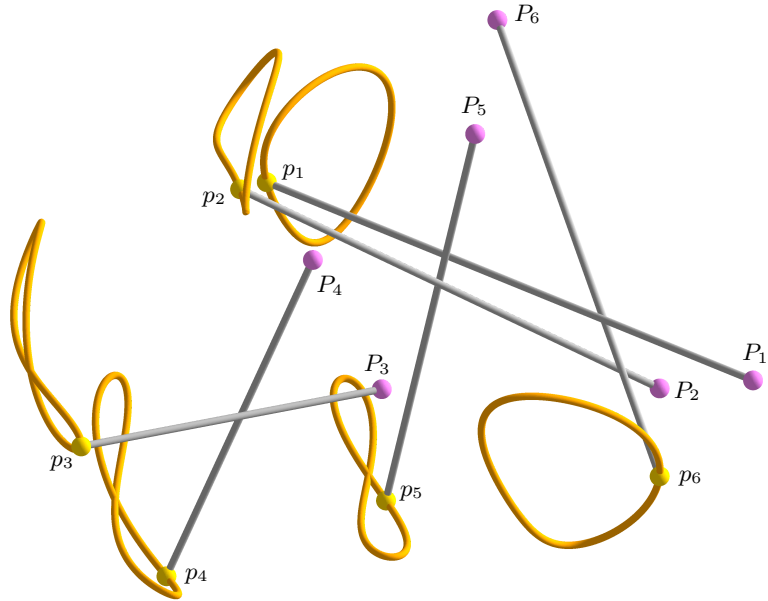
\begin{figure}[t] 
	\begin{center} 
		\begin{overpic}[width=100mm]{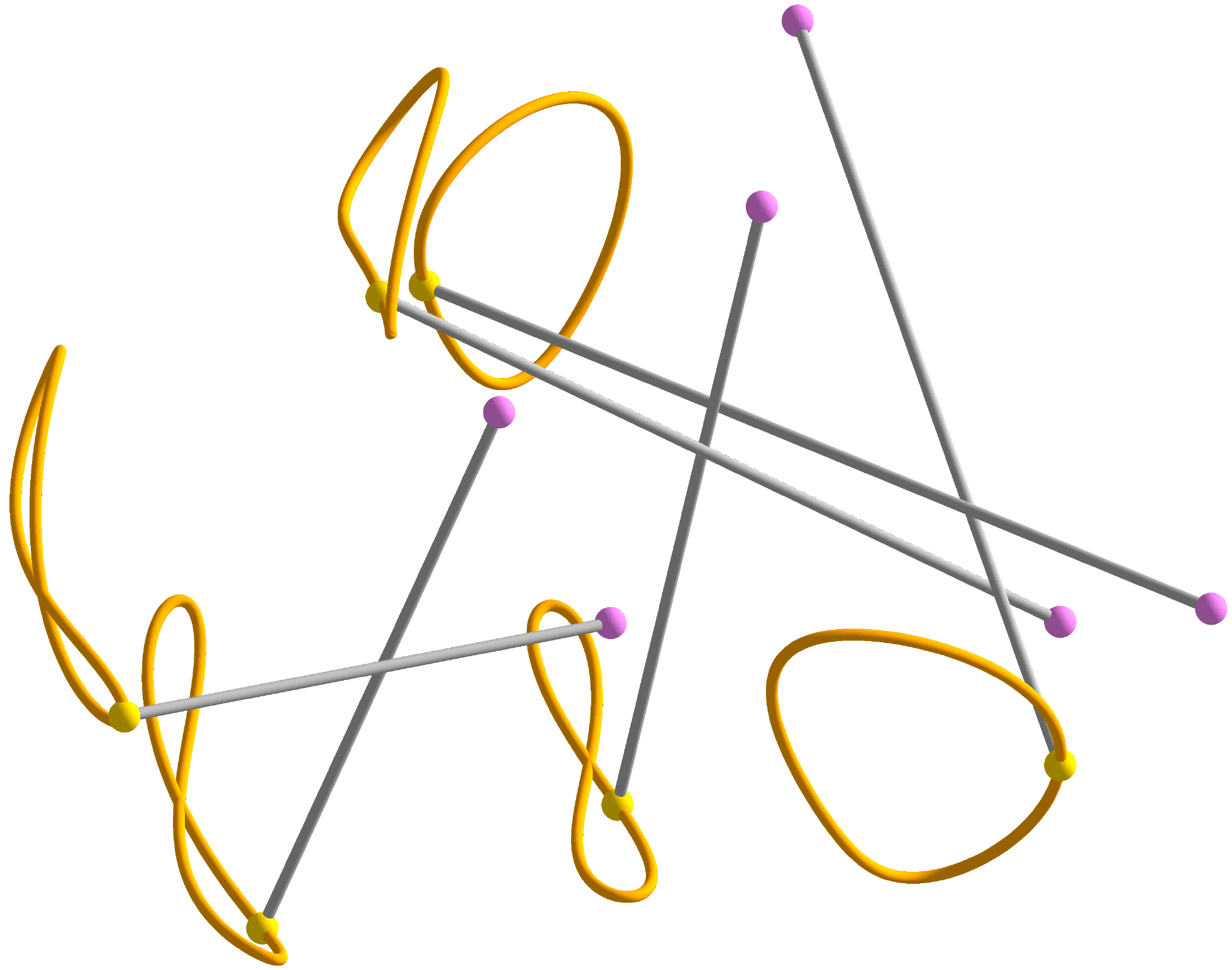}
			\begin{small}
			\put(23.7,3){$p_4$}
			\put(6,18){$p_3$}
			\put(26,53){$p_2$}
			\put(36.5,56.5){$p_1$}
			\put(52,14){$p_5$}
			\put(88.7,16){$p_6$}
			\put(41,41){$P_4$}
			\put(47.2,30.5){$P_3$}
			\put(60,64.8){$P_5$}
			\put(67,77){$P_6$}
			\put(88,27){$P_2$}
			\put(97,32){$P_1$}
			\end{small}
		\end{overpic} 
	\end{center}
	\caption{The hexapod in the configuration marked in Fig.~\ref{fig1} 
together with the trajectories corresponding to the red component of the 
self-motion. An animation of this self-motion can be downloaded from 
\url{www.geometrie.tuwien.ac.at/nawratil/liaison.gif}.}
	\label{fig2}
\end{figure} 

\bibliographystyle{amsalpha}
\providecommand{\bysame}{\leavevmode\hbox to3em{\hrulefill}\thinspace}

\end{document}